\documentclass{article} %
\usepackage{iclr2027_conference,times}   %
\usepackage[utf8]{inputenc}

\usepackage{amsmath,amsfonts,bm}

\def\eqref#1{equation~\ref{#1}}

\def\ceil#1{\lceil #1 \rceil}

\def\1{\bm{1}}

\def\eps{{\epsilon}}

\DeclareMathAlphabet{\mathsfit}{\encodingdefault}{\sfdefault}{m}{sl}
\SetMathAlphabet{\mathsfit}{bold}{\encodingdefault}{\sfdefault}{bx}{n}

\newcommand{\E}{\mathbb{E}}

\usepackage{amsmath,amssymb,amsthm}
\usepackage{tikz}
\usetikzlibrary{arrows.meta,decorations.pathreplacing,patterns,circuits.logic.US}
\usepackage{booktabs}
\usepackage{longtable}
\usepackage{xcolor}
\usepackage{textcomp}
\usepackage{graphicx}
\usepackage{wrapfig}
\usepackage{array}
\usepackage{url}
\usepackage{hyperref}
\usepackage[all]{hypcap}
\hypersetup{colorlinks=true, linkcolor=blue, citecolor=blue, urlcolor=blue}
\graphicspath{{figures/}}

\DeclareMathOperator{\suc}{succ}
\DeclareMathOperator{\decode}{dec}

\DeclareMathOperator{\bag}{bag}
\DeclareMathOperator{\TV}{TV}
\newcommand{\gap}{\Delta}
\newcommand{\TCz}{\mathsf{TC}^0}
\newcommand{\ACz}{\mathsf{AC}^0}

\newcommand{\NCone}{\mathsf{NC}^1}

\newcommand{\poly}{\mathrm{poly}}
\newcommand{\polylog}{\mathrm{polylog}}
\newcommand{\negl}{\mathrm{negl}}
\newcommand{\Unif}{\mathcal{U}}
\renewcommand{\E}{\mathbb{E}}
\newcommand{\ind}{\mathbf{1}}
\newcommand{\wt}[1]{\widetilde{#1}}
\renewcommand{\eps}{\varepsilon}
\newcommand{\PAR}{\mathrm{PAR}}
\newcommand{\WP}{\mathrm{WP}}

\theoremstyle{plain}
\newtheorem{theorem}{Theorem}[section]
\newtheorem{lemma}[theorem]{Lemma}
\newtheorem{proposition}[theorem]{Proposition}

\newtheorem{conjecture}[theorem]{Conjecture}
\newtheorem{fact}[theorem]{Fact}
\theoremstyle{definition}
\newtheorem{definition}[theorem]{Definition}
\newtheorem{assumption}[theorem]{Assumption}
\theoremstyle{remark}
\newtheorem{remark}[theorem]{Remark}
  \newcommand{\shal}{\mathrm{shallow}}
\title{Ceiling of a Task: When Can a Transformer Succeed Without Its Chain of Thought?}

\author{Jiashu He\thanks{Equal contribution.} \\
University of Pennsylvania \\
{\small\texttt{jiashuhe@seas.upenn.edu}}
\And
Jinxuan Fan\footnotemark[1] \\
University of Texas at Austin \\
{\small\texttt{jinxuan\_fan@utexas.edu}}
\And
Xiao Xiao \\
Yale University \\
{\small\texttt{xiao.xiao.xx244@yale.edu}}
\AND
Radu Marculescu \\
University of Texas at Austin \\
{\small\texttt{radum@utexas.edu}}
\And
Alejandro Ribeiro \\
University of Pennsylvania \\
{\small\texttt{aribeiro@seas.upenn.edu}}
}

\iclrfinalcopy
\begin{document}

{\setlength{\tabcolsep}{3pt}\maketitle}
\lhead{}\renewcommand{\headrulewidth}{0pt}
\suppressfloats[t]   %

\begin{abstract}
Reasoning models generate long chains of thought before they answer, yet it is debated whether the content of these chains does real computational work or is largely decorative. We study this question by viewing a transformer as a shallow circuit. One forward pass through a fixed number of layers has constant depth, so any procedure that runs the model a constant number of times is a shallow circuit. We call the best accuracy that a shallow circuit can reach on a task the \emph{ceiling} of the task, and a task is serial if its ceiling lies below one. We prove three results on serial tasks that hold for every transformer, no matter how it was trained. \emph{Necessity}: replacing the chain by anything that does not depend on its content, such as filler tokens or a restatement of the question, drives the accuracy down to the ceiling, and on a maximally serial task down to chance. \emph{Depth}: no shallow computation can write the chain of a model whose accuracy exceeds the ceiling, not even approximately. \emph{Locality}: the answer is one shallow pass away from the finished chain, so all of the serial reasoning happens in the chain. On word problems of finite groups, whose ceilings are known, small transformers trained from scratch, with or without reinforcement learning, attain the predicted numbers: chain-trained models solve every input length and fall to chance when the chain is erased, chainless models collapse to the ceiling as the input length grows, and open-weight reasoning models given the same problem in words return to the baseline without their chain. On MATH-500 and AIME, erasing the chain costs open reasoning models 0.52 to 0.82 accuracy, a sentence shuffle is harmless, and a token shuffle is as harmful as erasing; the same holds for checkpoints trained by GRPO with a correct or a random reward. The ceiling of a task therefore answers when a transformer can succeed without its chain of thought, no matter how the model was trained.
\end{abstract}

\section{Introduction}
\label{sec:intro}

Reasoning models produce long chains of thought, and longer chains are often correlated with better performance on hard problems. But it is unclear whether these chains do real computational work or are mostly decorative, as recent work shows that filler and pause tokens can reproduce some gains \citep{pfau2024, goyal2024pause}, and chains often fail to state the true reasons for answers \citep{turpin2023, chen2025reasoning}. Besides, experiments that erase, truncate, or shuffle chains \citep{lanham2023} are hard to interpret because it is unclear whether the accuracy drops are due to the chain or the model (answer head) reading it. This paper proves quantitative bounds on what any transformer can do without the content of its chain, no matter how it is trained.

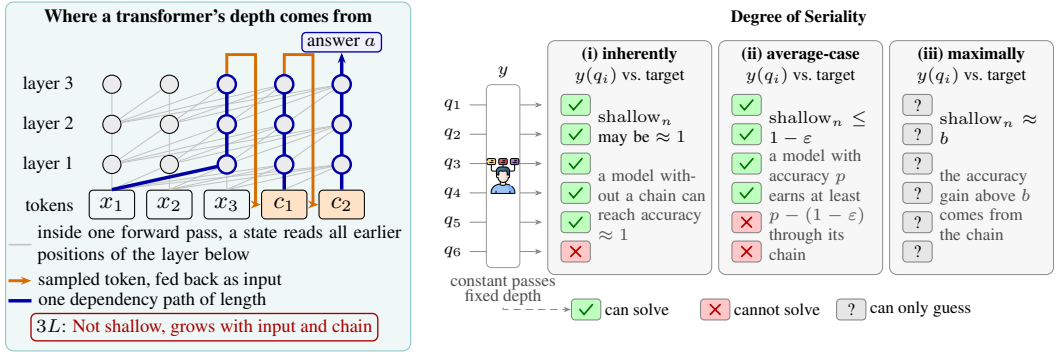
\begin{figure}[t]
\centering
\begin{minipage}[t]{0.40\linewidth}
\centering
\begin{tikzpicture}[x=0.76cm, y=0.53cm, >={Stealth[length=4pt]}, baseline=(current bounding box.north)]
\begin{scope}[tok/.style={draw, rounded corners=1.5pt, minimum width=0.6cm, minimum height=0.38cm, inner sep=0pt, font=\small},
  hid/.style={draw, circle, minimum size=0.24cm, inner sep=0pt, fill=black!8},
  dep/.style={black!24, line width=0.3pt},
  fb/.style={->, orange!85!black, line width=1.1pt},
  long/.style={blue!65!black, line width=1.3pt}]
  \draw[rounded corners=3pt, draw=teal!45, fill=teal!6] (-0.85,-3.65) rectangle (6.2,5.05);
  \node[font=\scriptsize\bfseries, anchor=north, inner sep=2pt] at (2.675,5.02) {Where a transformer's depth comes from};
  \foreach \i/\lab in {1/x_1, 2/x_2, 3/x_3} \node[tok] (t\i) at (\i,0) {$\lab$};
  \foreach \i/\lab in {4/c_1, 5/c_2} \node[tok, fill=orange!22] (t\i) at (\i,0) {$\lab$};
  \foreach \i in {1,...,5} \foreach \l in {1,2,3} \node[hid] (h\i\l) at (\i,\l) {};
  \foreach \t in {1,...,5} \foreach \s in {1,...,\t} {
    \draw[dep] (t\s.north) -- (h\t1);
    \draw[dep] (h\s1) -- (h\t2);
    \draw[dep] (h\s2) -- (h\t3); }
  \draw[long] (t1.north) -- (h31) -- (h32) -- (h33);
  \draw[long] (t4.north) -- (h41) -- (h42) -- (h43);
  \draw[long] (t5.north) -- (h51) -- (h52) -- (h53);
  \draw[fb] (h33.north) -- ++(0,0.5) -- ++(0.5,0) |- (t4.west);
  \draw[fb] (h43.north) -- ++(0,0.5) -- ++(0.5,0) |- (t5.west);
  \draw[->, long] (h53.north) -- ++(0,0.55) node[above, font=\scriptsize, text=black, draw=blue!65!black, line width=0.5pt, rounded corners=1.5pt, fill=blue!6, inner sep=2pt] {answer $a$};
  \foreach \n in {h31,h32,h33,h41,h42,h43,h51,h52,h53} \node[hid, draw=blue!65!black, line width=0.9pt] at (\n) {};
  \foreach \l in {1,2,3} \node[anchor=east, font=\scriptsize] at (0.5,\l) {layer \l};
  \node[anchor=east, font=\scriptsize] at (0.5,0) {tokens};
  \draw[dep, line width=0.6pt] (-0.8,-1.0) -- (-0.4,-1.0) node[right, font=\scriptsize, text=black, align=left, inner sep=2pt, text width=4.9cm] {inside one forward pass, a state reads all earlier positions of the layer below};
  \draw[fb] (-0.8,-1.9) -- (-0.4,-1.9) node[right, font=\scriptsize, text=black, inner sep=2pt] {sampled token, fed back as input};
  \draw[long] (-0.8,-2.4) -- (-0.4,-2.4) node[right, font=\scriptsize, text=black, inner sep=2pt] (lg3) {one dependency path of length};
  \node[font=\scriptsize, draw=red!65!black, rounded corners=2pt, line width=0.5pt, inner sep=2pt, anchor=north west] (threeL) at (-0.42,-2.7) {$3L$: \textcolor{red!65!black}{Not shallow, grows with input and chain}};
\end{scope}
\end{tikzpicture}
\end{minipage}\hfill
\begin{minipage}[t]{0.58\linewidth}
\centering
\resizebox{\linewidth}{!}{%
\begin{tikzpicture}[>={Stealth[length=3pt]}, font=\footnotesize, baseline=(current bounding box.north),
  ycell/.style={draw=black!50, rounded corners=1.5pt, minimum width=0.42cm, minimum height=0.3cm, inner sep=0pt, font=\scriptsize},
  ok/.style={ycell, fill=green!25}, bad/.style={ycell, fill=red!22}, guess/.style={ycell, fill=black!10},
  lab/.style={font=\scriptsize, inner sep=1pt},
  ttl/.style={font=\scriptsize\bfseries, inner sep=1pt}]
\newcommand{\tick}[2]{\draw[green!45!black, line width=0.7pt, line cap=round] (#1-0.09,#2) -- (#1-0.03,#2-0.07) -- (#1+0.09,#2+0.07);}
\newcommand{\cross}[2]{\draw[red!70!black, line width=0.7pt, line cap=round] (#1-0.07,#2-0.07) -- (#1+0.07,#2+0.07) (#1-0.07,#2+0.07) -- (#1+0.07,#2-0.07);}
\foreach \k in {1,...,6} {
  \pgfmathsetmacro{\r}{3.28-0.42*(\k-1)}
  \node[lab, anchor=east] at (0.42,\r) {$q_\k$};
  \draw[->, black!45, line width=0.4pt] (0.48,\r) -- (1.52,\r);
}
\draw[rounded corners=2pt, fill=white, draw=black!40] (0.72,0.95) rectangle (1.2,3.57);
\node[inner sep=0pt] at (0.96,2.26) {\includegraphics[height=0.5cm]{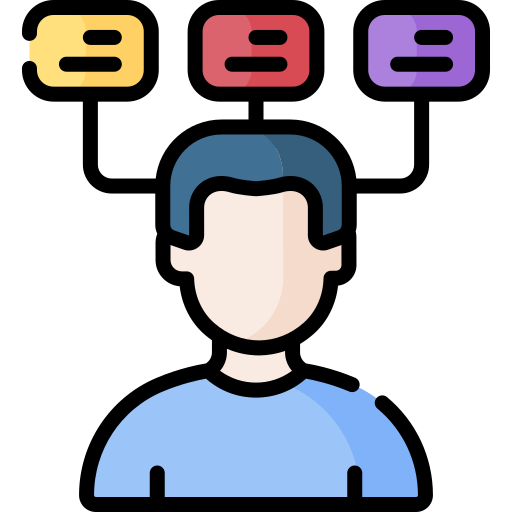}};
\node[lab] at (0.96,3.75) {$y$};
\node[lab, text=black!70, align=center, anchor=north] (cp) at (0.96,0.87) {constant passes\\[-2pt] fixed depth};
\node[font=\scriptsize\bfseries, anchor=south] at (5.2,4.3) {Degree of Seriality};
\foreach \fx/\ttl/\pattern/\ceil/\cons in {
  1.6/{(i) inherently}/{1,1,1,1,1,2}/{$\shal_n$ may be $\approx 1$}/{a model without a chain can reach accuracy $\approx 1$},
  4.05/{(ii) average-case}/{1,1,1,1,2,2}/{$\shal_n \le 1 - \varepsilon$}/{a model with accuracy $p$ earns at least $p - (1-\varepsilon)$ through its chain},
  6.5/{(iii) maximally}/{3,3,3,3,3,3}/{$\shal_n \approx b$}/{the accuracy gain above $b$ comes from the chain}} {
  \draw[rounded corners=3pt, draw=black!45, fill=black!3] (\fx,0.85) rectangle (\fx+2.35,4.2);
  \node[ttl, anchor=north] at (\fx+1.175,4.15) {\ttl};
  \node[lab, anchor=north] at (\fx+1.175,3.86) {$y(q_i)$ vs.\ target};
  \foreach \pat [count=\k] in \pattern {
    \pgfmathsetmacro{\r}{3.28-0.42*(\k-1)}
    \ifnum\pat=1 \node[ok] at (\fx+0.4,\r) {}; \tick{\fx+0.4}{\r} \fi
    \ifnum\pat=2 \node[bad] at (\fx+0.4,\r) {}; \cross{\fx+0.4}{\r} \fi
    \ifnum\pat=3 \node[guess] at (\fx+0.4,\r) {?}; \fi
  }
  \node[lab, align=left, text width=1.55cm, anchor=west] at (\fx+0.68,2.95) {\ceil};
  \node[lab, align=left, text=black!70, text width=1.55cm, anchor=west] at (\fx+0.68,1.85) {\cons};
}
\node[ok] (lg1) at (2.15,0.35) {}; \tick{2.15}{0.35} \node[lab, anchor=west] at (2.38,0.35) {can solve};
\node[bad] at (4.0,0.35) {}; \cross{4.0}{0.35} \node[lab, anchor=west] at (4.23,0.35) {cannot solve};
\node[guess] at (5.95,0.35) {?}; \node[lab, anchor=west] at (6.18,0.35) {can only guess};
\draw[->, black!45, densely dashed, line width=0.4pt] (cp.south) |- (lg1.west);
\end{tikzpicture}}
\end{minipage}
\caption{\emph{Left:} a transformer with $L = 3$ layers on three input tokens and two generated ones. Inside one forward pass a hidden state reads only earlier positions of the layer below, so information travels at most $L$ steps. A transformer is therefore \emph{shallow} when its output length is constant and not when it grows with the input. \emph{Right:} the three degrees of seriality of Definition~\ref{def:serial} for a task: a collection of inputs $q_i$ with targets $y(q_i)$.}
\label{fig:grid}
\vspace{-0.4cm}
\end{figure}

To study how CoT improves information processing in LLMs when handling complex tasks, we model one forward pass of a transformer with $L$ layers as a calculation conducted by a circuit of constant depth $L$, where ``constant'' means that the amount of parallelizable computation does not depend on the input size, and we call such a circuit a \emph{shallow circuit}. As shown in Figure~\ref{fig:grid} (left), increasing the input length and context window expands the circuit's width, but never increases the information processing depth. Depth can only be increased by producing a token and then feeding it back in as part of the input. As a result, generating $T$ output tokens causes the input information to be processed through $T$ sequential forward passes, rather than a single pass when there is no CoT. The intermediate tokens are the links that carry information forward from one pass to the next across the $T$ passes.

This is what makes reasoning powerful: with polynomially many intermediate tokens, a transformer can solve problems that no single forward pass can \citep{li2024cot, merrillsabharwal2024cot}. Existing results of this type, based on the circuit perspective on transformers \citep{merrillsabharwal2023parallelism, merrillsabharwal2025depth, amiri2025lower}, establish that for a given class of problems there exists a transformer with some chain length that solves it. We ask a different question, about one transformer that has already been trained by any algorithm: \emph{what can it do at test time without the content of its chain?}

A fixed transformer run a constant number of times is itself a shallow circuit (Lemma~\ref{lem:pipeline}), so its accuracy is bounded by the best accuracy any shallow circuit reaches on the task, the task's \emph{ceiling} $\shal_n$. A task is defined as \emph{average-case serial} if its ceiling is bounded away from one and \emph{maximally serial} if its ceiling is the guessing baseline (Section~\ref{sec:prelim}); for the word problem of the group $A_5$, the ceiling is the guessing baseline by a result of \citet{milesviola2013}, which makes it the test case of Section~\ref{sec:experiments}. Applying the lemma at the three places of Figure~\ref{fig:pipeline} (right) gives three theorems that hold for every transformer, however it was trained. \emph{Necessity} (Section~\ref{sec:necessity}): replace the chain by anything that does not depend on its content, such as filler tokens or a restatement of the question, and the accuracy falls to the ceiling, or to chance on a maximally serial task. \emph{Depth} (Section~\ref{sec:depth}): a chain that lifts accuracy above the ceiling cannot be written by any shallow computation, not even approximately, so its tokens must build on one another over a number of steps that grows with the input. \emph{Locality} (Section~\ref{sec:locality}): the answer head is one forward pass, so once the chain ends the answer is readable from it by a shallow circuit; all the serial work sits in the chain. None of this requires the chain to be faithful or human-readable. Appendix~\ref{sec:discussion} discusses what the theorems say about current reasoning models, from one-pass models to process reward models.

We test the theory at two scales (Section~\ref{sec:experiments}). On word problems of finite groups, whose ceilings are known, small transformers trained from scratch, with or without reinforcement learning, match the predicted numbers: with a chain they reach accuracy $1$ at every input length, and fall to chance when the chain is erased. Without a chain they collapse to the ceiling once the input outgrows their depth. Open-weight reasoning models given the same word problem in words also return to the baseline without their chain, and a chain of bounded length helps them less as the input grows. On AIME and MATH-500, where the ceiling is unknown, open reasoning models with 1.5B to 32B parameters lose 0.52 to 0.82 accuracy when the chain is erased. With only a prefix of the chain, the remaining accuracy depends on how often the kept prefix already states the answer. Their accuracy survives a shuffle of the sentences of the chain but not of its tokens, since a shallow decoder can still read the answer from the bag of sentences, and the same pattern holds for checkpoints trained by GRPO with a correct or a random reward. Together, they reveal that without its chain a transformer keeps exactly what one pass can compute, and everything above the ceiling of the task comes from the chain, whether the model was trained by supervision or by reinforcement learning.

Our contributions are summarized as follows.
\begin{itemize}
\item A framework that measures a task by its ceiling, the best accuracy of a shallow circuit. The ceiling bounds the accuracy that reasoning LLMs, small transformers, one-pass transformers and latent reasoners alike can reach without a chain. It also grades tasks by how much of the accuracy on them has to come from serial computation, which we call degrees of seriality (Section~\ref{sec:prelim}). We extend the worst-case notion of \citet{li2024cot} to the average case.
\item Three theorems on chains of thought, Necessity, Depth, and Locality, that hold for every transformer at test time regardless of how it was trained and bound what a trained transformer can do without the content of its chain (Sections~\ref{sec:necessity} to~\ref{sec:locality}).
\item Experiments on tasks with a known ceiling and on math reasoning benchmarks: small transformers trained from scratch, with or without RL, and open-weight reasoning models land on the predicted ceiling without their chains, and interventions on open reasoning models on mathematical benchmarks, including erasing, truncating, and shuffling the chain, all behave exactly as the theorems predict (Section~\ref{sec:experiments}).
\end{itemize}
\section{Transformers as shallow circuits}
\label{sec:prelim}

\newcommand{\icCeil}[3]{\node[inner sep=0pt, fill=white] at (#1,#2) {\includegraphics[height=#3\dimexpr0.5cm\relax]{figures/icons/shallow.png}};}
\newcommand{\icSucc}[3]{\node[inner sep=0pt, fill=white] at (#1,#2) {\includegraphics[height=#3\dimexpr0.5cm\relax]{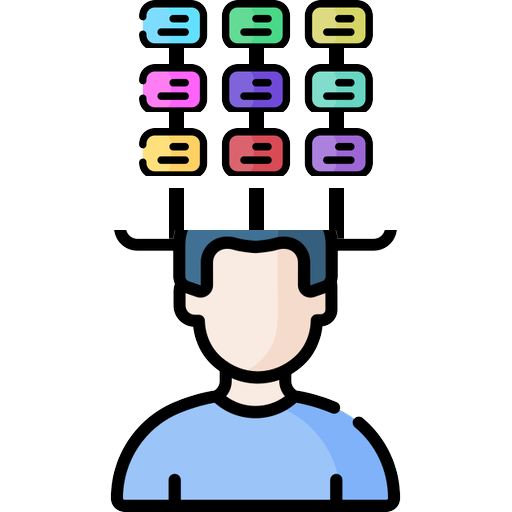}};}
\newcommand{\icDec}[3]{\node[inner sep=0pt, fill=white] at (#1,#2) {\includegraphics[height=#3\dimexpr0.5cm\relax]{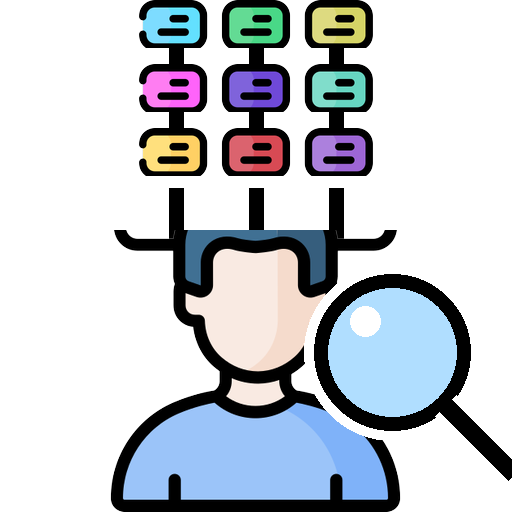}};}
\newcommand{\icErased}[3]{\node[inner sep=0pt, fill=white] at (#1,#2) {\includegraphics[height=#3\dimexpr0.5cm\relax]{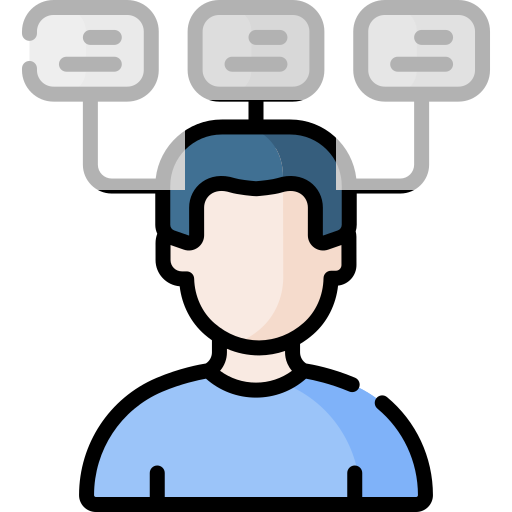}};}
\newcommand{\icDie}[3]{\node[inner sep=0pt, fill=white] at (#1,#2) {\includegraphics[height=#3\dimexpr0.5cm\relax]{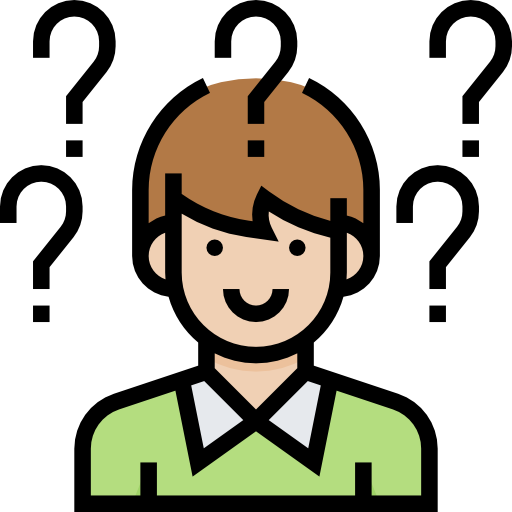}};}
\newcommand{\icCeilI}{\raisebox{-0.25\height}{\includegraphics[height=0.42cm]{figures/icons/shallow.png}}}
\newcommand{\icSuccI}{\raisebox{-0.25\height}{\includegraphics[height=0.42cm]{figures/icons/reasoning.png}}}
\newcommand{\icDecI}{\raisebox{-0.25\height}{\includegraphics[height=0.42cm]{figures/icons/decode.png}}}
\newcommand{\icDieI}{\raisebox{-0.25\height}{\includegraphics[height=0.42cm]{figures/icons/guess.png}}}
 
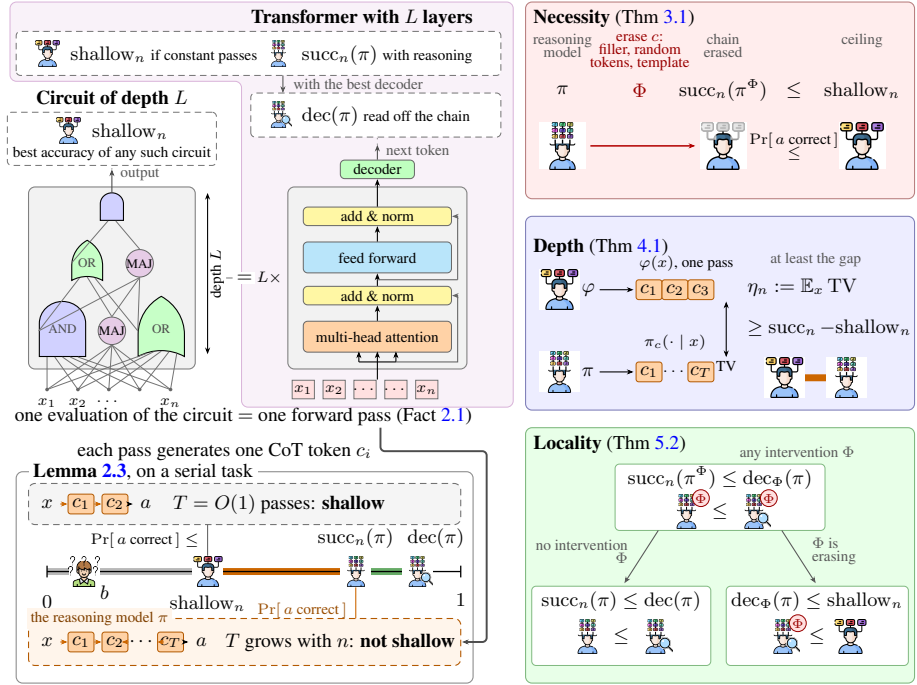
\begin{figure}[t]
\centering
\resizebox{0.86\linewidth}{!}{%
\begin{tikzpicture}[yscale=0.95, circuit logic US, circuit symbol unit=3.6pt, >={Stealth[length=3.2pt]}, font=\small,
  every circuit symbol/.style={line width=0.5pt, fill=white},
  wire/.style={black!55, line width=0.35pt},
  tok/.style={draw=orange!70!black, rounded corners=1.2pt, fill=orange!35, minimum width=0.42cm, minimum height=0.32cm, inner sep=1pt, font=\footnotesize},
  gtok/.style={tok, draw=black!50, fill=black!12},
  htok/.style={tok, draw=black!50, fill=black!4, pattern=north east lines, pattern color=black!35},
  xbox/.style={draw=black!55, fill=red!12, minimum width=0.36cm, minimum height=0.3cm, inner sep=0pt, font=\scriptsize},
  blk/.style={draw=black!60, rounded corners=2pt, minimum width=1.9cm, inner sep=0pt, font=\scriptsize, align=center},
  pass/.style={draw=black!60, rounded corners=1.5pt, fill=black!5, minimum width=0.46cm, minimum height=0.46cm, inner sep=0pt},
  arr/.style={->, line width=0.5pt},
  oarr/.style={->, orange!80!black, line width=0.6pt},
  eq/.style={black!50, densely dashed, line width=0.5pt},
  ttl/.style={font=\footnotesize\bfseries, inner sep=1pt},
  note/.style={font=\scriptsize, text=black!65, inner sep=1pt},
  cell/.style={font=\footnotesize, anchor=north west, align=left, inner sep=2pt, text width=2.3cm},
  panel/.style={rounded corners=3pt, line width=0.5pt}]
\newcommand{\minipass}[3]{\node[pass] (#1) at (#2,#3) {};
  \foreach \i in {-1,0,1} \foreach \j in {-0.5,0.5} \draw[black!35, line width=0.3pt] (#2+0.13*\i, #3-0.15) -- (#2+0.13*\j, #3);
  \foreach \j in {-0.5,0.5} \draw[black!35, line width=0.3pt] (#2+0.13*\j, #3) -- (#2, #3+0.15);
  \foreach \i in {-1,0,1} \fill[black!50] (#2+0.13*\i, #3-0.15) circle (0.027);
  \foreach \i in {-0.5,0.5} \fill[black!50] (#2+0.13*\i, #3) circle (0.027);
  \fill[black!50] (#2, #3+0.15) circle (0.027);}
\node[ttl] at (1.65,10.2) {Circuit of depth $L$};
\draw[panel, draw=black!50, fill=black!5] (0.3,5.60) rectangle (3.0,8.70);
\foreach \j in {0,...,4} \coordinate (in\j) at (0.6+0.5*\j,5.25);
\foreach \j/\lab in {0/$x_1$, 1/$x_2$, 2/$\cdots$, 4/$x_n$} \node[font=\scriptsize, anchor=north, inner sep=1pt] at (0.6+0.5*\j,5.15) {\lab};
\node[and gate, inputs={nnnnn}, point up, fill=blue!18] (g1) at (0.85,6.25) {};
\node[draw=black!60, circle, fill=violet!20, minimum size=0.44cm, inner sep=0pt, font=\tiny] (g2) at (1.65,6.25) {MAJ};
\node[or gate, inputs={nnnnn}, point up, fill=green!22] (g3) at (2.45,6.25) {};
\foreach \j in {1,...,5} { \pgfmathtruncatemacro{\k}{\j-1} \draw[wire] (in\k) -- (g1.input \j); \draw[wire] (in\k) -- (g3.input \j); \draw[wire] (in\k) -- (g2.south); }
\foreach \j in {0,...,4} \fill[black!60] (in\j) circle (0.03);
\node[or gate, inputs={nnn}, point up, fill=green!22] (h1) at (1.25,7.40) {};
\node[draw=black!60, circle, fill=violet!20, minimum size=0.44cm, inner sep=0pt, font=\tiny] (h2) at (2.1,7.40) {MAJ};
\foreach \g/\i in {g1/1, g2/2, g3/3} { \draw[wire] (\g.north) -- (h1.input \i); \draw[wire] (\g.north) -- (h2.south); }
\node[and gate, inputs={nn}, point up, fill=blue!18] (o1) at (1.65,8.35) {};
\draw[wire] (h1.north) -- (o1.input 1); \draw[wire] (h2.north) -- (o1.input 2);
\node[font=\tiny, text=black!70] at ([yshift=0.03cm]g1.center) {AND};
\node[font=\tiny, text=black!70] at ([yshift=0.03cm]g3.center) {OR};
\node[font=\tiny, text=black!70] at ([yshift=0.02cm]h1.center) {OR};
\draw[arr, black!70] (o1.output) -- (1.65,9.00);
\node[note, anchor=west] at (1.75,8.92) {output};
\draw[<->, line width=0.5pt] (3.2,5.85) -- (3.2,8.60) node[midway, font=\scriptsize, inner sep=1pt, rotate=90, anchor=north] {depth $L$};
\draw[rounded corners=3pt, draw=violet!30, fill=violet!5] (0.0,10.5) -- (0.0,11.85) -- (8.15,11.85) -- (8.15,4.95) -- (3.75,4.95) -- (3.75,10.5) -- cycle;
\node[ttl] at (5.7,11.58) {Transformer with $L$ layers};
\foreach \j/\lab in {0/$x_1$, 1/$x_2$, 2/, 3/, 4/$x_n$} \node[xbox] (X\j) at (4.75+0.5*\j,5.25) {\lab};
\node[font=\scriptsize] at (4.75+0.5*2,5.25) {$\cdots$};
\node[font=\scriptsize] at (4.75+0.5*3,5.25) {$\cdots$};
\draw[panel, draw=black!50, fill=black!5] (4.5,5.6) rectangle (7.4,8.7);
\node[blk, fill=orange!35, minimum height=0.5cm, minimum width=2.4cm] (att) at (5.95,6.15) {multi-head attention};
\node[blk, fill=yellow!45, minimum height=0.3cm, minimum width=2.4cm] (an1) at (5.95,6.85) {add \& norm};
\node[blk, fill=cyan!25, minimum height=0.5cm, minimum width=2.4cm] (ffn) at (5.95,7.5) {feed forward};
\node[blk, fill=yellow!45, minimum height=0.3cm, minimum width=2.4cm] (an2) at (5.95,8.2) {add \& norm};
\draw[line width=0.5pt] (5.95,5.42) -- (5.95,5.72);
\draw[line width=0.5pt] (5.65,5.72) -- (6.25,5.72);
\foreach \x in {5.65,5.95,6.25} \draw[arr] (\x,5.72) -- (\x,5.9);
\draw[arr] (att) -- (an1); \draw[arr] (an1) -- (ffn); \draw[arr] (ffn) -- (an2);
\draw[arr, black!60] (6.25,5.72) -- (7.25,5.72) |- (an1.east);
\draw[black!60, line width=0.5pt] (5.95,7.12) -- (7.25,7.12); \draw[arr, black!60] (7.25,7.12) |- (an2.east);
\node[font=\scriptsize, anchor=east, inner sep=1pt] at (4.45,7.15) {$L\times$};
\node[blk, fill=green!25, minimum height=0.3cm, minimum width=1.2cm] (sm) at (5.95,9.0) {decoder};
\draw[arr] (an2) -- (sm);
\draw[arr, black!70] (sm) -- (5.95,9.42);
\node[note, anchor=west] at (6.05,9.33) {next token};
 
\node[draw=black!55, densely dashed, rounded corners=2pt, fill=white, align=center, inner sep=2.5pt, font=\footnotesize] (bC) at (1.65,9.52) {\icCeilI\; $\shal_n$\\[-1pt] {\scriptsize best accuracy of any such circuit}};
\draw[black!50, dotted, line width=0.5pt] (1.65,9.0) -- (bC.south);
\draw[draw=black!55, densely dashed, rounded corners=2pt, fill=white] (0.1,10.6) rectangle (7.95,11.35);
\node[font=\footnotesize, anchor=west, inner sep=1pt] at (0.35,10.97) {\icCeilI\; $\shal_n$ {\scriptsize if constant passes}};
\node[font=\footnotesize, anchor=west, inner sep=1pt] (bS) at (4.1,10.97) {\icSuccI\; $\suc_n(\pi)$ {\scriptsize with reasoning}};
\draw[arr, black!65] (4.42,10.6) -- (4.42,10.3);
\node[font=\scriptsize, text=black!65, anchor=west, inner sep=1pt] at (4.55,10.45) {with the best decoder};
\draw[draw=black!55, densely dashed, rounded corners=2pt, fill=white] (3.9,9.55) rectangle (7.95,10.3);
\node[font=\footnotesize, anchor=west, inner sep=1pt] at (4.1,9.93) {\icDecI\; $\decode(\pi)$ {\scriptsize read off the chain}};
\draw[black!50, dotted, line width=0.5pt] (5.95,9.42) -- (5.95,9.55);
\draw[eq] (3.5,7.15) -- (4.05,7.15); \node[font=\small, fill=white, inner sep=1pt] at (3.8,7.15) {$=$};
\node[font=\footnotesize, align=center] at (3.78,4.80) {one evaluation of the circuit $=$ one forward pass (Fact~\ref{fact:onepass})};
 
\draw[panel, draw=black!45, fill=white] (0.1,0.25) rectangle (7.5,3.85);
\node[font=\footnotesize, anchor=west, fill=white, inner sep=2pt] at (0.3,3.85) {\textbf{Lemma~\ref{lem:pipeline}}, on a serial task};
\draw[panel, draw=black!55, densely dashed, fill=black!4] (0.3,2.95) rectangle (7.3,3.65);
\node[font=\footnotesize] (s0) at (0.6,3.3) {$x$};
\node[tok] (s1) at (1.15,3.3) {$c_1$}; \node[tok] (s2) at (1.7,3.3) {$c_2$};
\node[font=\footnotesize] (s3) at (2.2,3.3) {$a$};
\draw[oarr] (s0) -- (s1); \draw[oarr] (s1) -- (s2); \draw[arr] (s2) -- (s3);
\node[font=\footnotesize, anchor=west, inner sep=1pt] at (2.6,3.3) {$T = O(1)$ passes: \textbf{shallow}};
\draw[black!28, line width=3.2pt] (0.6,2.15) -- (3.2,2.15);
\draw[orange!80!black, line width=3.2pt] (3.2,2.15) -- (5.6,2.15);
\draw[green!45!black!60, line width=3.2pt] (5.6,2.15) -- (6.6,2.15);
\draw[line width=0.6pt] (0.6,2.15) -- (7.3,2.15);
\foreach \x in {0.6, 7.3} \draw[line width=0.6pt] (\x,2.05) -- (\x,2.25);
\icDie{1.2}{2.15}{1} \icCeil{3.2}{2.15}{1} \icSucc{5.6}{2.15}{1} \icDec{6.6}{2.15}{1}
\node[font=\footnotesize, anchor=north] at (0.6,1.88) {$0$};
\node[font=\footnotesize, anchor=north west, inner sep=1pt] at (1.42,1.95) {$b$};
\node[font=\footnotesize, anchor=north] at (3.2,1.88) {$\shal_n$};
\node[font=\footnotesize, anchor=south] at (5.6,2.42) {$\suc_n(\pi)$};
\node[font=\footnotesize, anchor=north] at (7.3,1.9) {$1$};
\node[font=\footnotesize, anchor=south] at (6.9,2.42) {$\decode(\pi)$};
\draw[black!60, line width=0.5pt] (3.2,2.95) -- (3.2,2.42);
\node[font=\scriptsize, anchor=east, inner sep=1pt] at (3.1,2.68) {$\Pr[\,a \text{ correct}\,] \le$};
\draw[orange!80!black, line width=0.5pt] (5.6,1.35) -- (5.6,1.88);
\node[font=\scriptsize, anchor=east, inner sep=1pt, text=orange!60!black] at (5.5,1.5) {$\Pr[\,a \text{ correct}\,]$};
\draw[panel, draw=orange!70!black, densely dashed, fill=orange!8] (0.3,0.55) rectangle (7.3,1.35);
\node[font=\scriptsize, fill=orange!8, inner sep=1.5pt, text=orange!50!black] at (1.45,1.35) {the reasoning model $\pi$};
\node[font=\footnotesize] (u0) at (0.6,0.95) {$x$};
\node[tok] (u1) at (1.15,0.95) {$c_1$}; \node[tok] (u2) at (1.7,0.95) {$c_2$};
\node[font=\footnotesize] (ud) at (2.15,0.95) {$\cdots$};
\node[tok] (uT) at (2.6,0.95) {$c_T$};
\node[font=\footnotesize] (ua) at (3.1,0.95) {$a$};
\draw[oarr] (u0) -- (u1); \draw[oarr] (u1) -- (u2); \draw[oarr] (u2) -- (ud); \draw[oarr] (ud) -- (uT); \draw[arr] (uT) -- (ua);
\node[font=\footnotesize, anchor=west, inner sep=1pt] at (3.45,0.95) {$T$ grows with $n$: \textbf{not shallow}};
\draw[->, line width=0.9pt, black!70, rounded corners=2pt] (6.0,4.6) -- (6.0,4.2) -- (7.72,4.2) -- (7.72,0.95) -- (7.32,0.95);
\node[font=\footnotesize, anchor=east, inner sep=1pt] at (5.9,4.2) {each pass generates one CoT token $c_i$};
\draw[panel, draw=red!45!black!50, fill=red!7] (8.35,8.5) rectangle (14.65,11.85);
\node[font=\footnotesize, anchor=north west, inner sep=2pt] at (8.4,11.80) {\textbf{Necessity} (Thm~\ref{thm:necessity})};
\node[note, anchor=north, align=center] at (8.95,11.35) {reasoning\\[-2pt] model};
\node[note, anchor=north, align=center, text=red!60!black] at (10.2,11.35) {erase $c$:\\[-2pt] filler, random\\[-2pt] tokens, template};
\node[note, anchor=north, align=center] at (11.55,11.35) {chain\\[-2pt] erased};
\node[note, anchor=north, align=center] at (13.8,11.35) {ceiling};
\node[font=\footnotesize] at (8.95,10.35) {$\pi$};
\node[font=\footnotesize, text=red!60!black] at (10.2,10.35) {$\Phi$};
\node[font=\footnotesize] at (11.55,10.35) {$\suc_n(\pi^{\Phi})$};
\node[font=\footnotesize] at (12.7,10.35) {$\le$};
\node[font=\footnotesize] at (13.8,10.35) {$\shal_n$};
\icSucc{8.95}{9.40}{1.5}
\draw[arr, red!70!black, line width=0.7pt] (9.4,9.40) -- (11.1,9.40);
\icErased{11.55}{9.40}{1.5}
\node[font=\scriptsize, align=center] at (12.7,9.40) {$\Pr[\,a \text{ correct}\,]$\\[-2pt] $\le$};
\icCeil{13.8}{9.40}{1.5}
\draw[panel, draw=blue!45!black!50, fill=blue!7] (8.35,4.9) rectangle (14.65,8.2);
\node[font=\footnotesize, anchor=north west, inner sep=2pt] at (8.4,7.90) {\textbf{Depth} (Thm~\ref{thm:depth})};
\node[note, anchor=north] at (13.05,7.62) {at least the gap};
\icCeil{8.9}{6.95}{1.3}
\node[font=\footnotesize, anchor=west, inner sep=1pt] at (9.22,6.95) {$\varphi$};
\draw[arr] (9.55,6.95) -- (10.1,6.95);
\node[tok] at (10.35,6.95) {$c_1$}; \node[tok] at (10.77,6.95) {$c_2$}; \node[tok] at (11.19,6.95) {$c_3$};
\node[font=\scriptsize, anchor=south west, inner sep=1pt] at (10.15,7.22) {$\varphi(x)$, one pass};
\icSucc{8.9}{5.55}{1.3}
\node[font=\footnotesize, anchor=west, inner sep=1pt] at (9.22,5.55) {$\pi$};
\draw[arr] (9.55,5.55) -- (10.1,5.55);
\node[tok] at (10.35,5.55) {$c_1$}; \node[font=\scriptsize] at (10.77,5.55) {$\cdots$}; \node[tok] at (11.19,5.55) {$c_T$};
\node[font=\scriptsize, anchor=south] at (10.77,5.80) {$\pi_c(\cdot \mid x)$};
\draw[<->, line width=0.5pt] (11.6,5.80) -- (11.6,6.75) node[at start, below, font=\scriptsize, inner sep=1.5pt] {TV};
\node[font=\footnotesize, anchor=west, inner sep=1pt] at (11.9,7.00) {$\eta_n := \mathbb{E}_x\,\mathrm{TV}$};
\node[font=\footnotesize, anchor=west, inner sep=1pt] at (11.9,6.30) {$\ge \suc_n - \shal_n$};
\draw[orange!80!black, line width=3.2pt] (12.55,5.45) -- (13.55,5.45); \icCeil{12.55}{5.45}{1.3} \icSucc{13.55}{5.45}{1.3}
\draw[panel, draw=green!45!black!50, fill=green!9] (8.35,0.25) rectangle (14.65,4.6);
\node[font=\footnotesize, anchor=north west, inner sep=2pt] at (8.4,4.55) {\textbf{Locality} (Thm~\ref{thm:locality})};
\newcommand{\phibadge}[2]{\node[circle, draw=red!60!black, fill=red!12, inner sep=0.4pt, font=\tiny, text=red!60!black] at (#1+0.24,#2+0.22) {$\Phi$};}
\node[note, anchor=south] at (12.75,4.0) {any intervention $\Phi$};
\draw[rounded corners=2pt, draw=green!45!black!50, fill=white] (9.8,2.85) rectangle (13.2,3.98);
\node[font=\footnotesize] at (11.5,3.75) {$\suc_n(\pi^{\Phi}) \le \decode_{\Phi}(\pi)$};
\icSucc{10.95}{3.18}{1.1} \phibadge{10.95}{3.18}
\node[font=\footnotesize] at (11.5,3.18) {$\le$};
\icDec{12.05}{3.18}{1.1} \phibadge{12.05}{3.18}
\draw[arr, black!65] (10.5,2.85) -- (9.95,1.95);
\node[note, anchor=east, align=right] at (10.05,2.5) {no intervention\\[-2pt] $\Phi$};
\draw[rounded corners=2pt, draw=green!45!black!50, fill=white] (8.45,0.55) rectangle (11.35,1.9);
\node[font=\footnotesize] at (9.9,1.65) {$\suc_n(\pi) \le \decode(\pi)$};
\icSucc{9.35}{1.05}{1.1}
\node[font=\footnotesize] at (9.9,1.05) {$\le$};
\icDec{10.45}{1.05}{1.1}
\draw[arr, black!65] (12.5,2.85) -- (13.05,1.95);
\node[note, anchor=west, align=left] at (12.95,2.5) {$\Phi$ is\\[-2pt] erasing};
\draw[rounded corners=2pt, draw=green!45!black!50, fill=white] (11.6,0.55) rectangle (14.55,1.9);
\node[font=\footnotesize] at (13.07,1.65) {$\decode_{\Phi}(\pi) \le \shal_n$};
\icDec{12.52}{1.05}{1.1} \phibadge{12.52}{1.05}
\node[font=\footnotesize] at (13.07,1.05) {$\le$};
\icCeil{13.62}{1.05}{1.1}
\end{tikzpicture}}

\caption{\emph{Top left:} one forward pass of a transformer with $L$ layers can be viewed as a circuit of constant depth $L$ with unbounded fan-in (Fact~\ref{fact:onepass}). The ceiling $\shal_n$ (\protect\icCeilI) is the best accuracy of any such shallow circuit, $\suc_n(\pi)$ (\protect\icSuccI) is the accuracy of the reasoning model, and $\decode(\pi)$ (\protect\icDecI) is what the best shallow decoder reads off the finished chain. \emph{Bottom left (Lemma~\ref{lem:pipeline}):} a constant number of passes is still shallow and lands below the ceiling; only a chain that grows with the input can rise above it. \emph{Right:} for every transformer, erasing the chain sends the accuracy back to the ceiling (Necessity), no shallow computation could write a chain close to the model's reasoning chain (Depth), and the finished chain already makes the answer readable in one pass (Locality).}
\label{fig:pipeline}
\vspace{-0.3cm}
\end{figure}

\paragraph{Circuits and depth.}
We use $n$ to denote the length of the input of a circuit, and $\poly(n)$ is the set of all quantities bounded by a polynomial in $n$. The \emph{depth} of a Boolean circuit is the length of the longest path from an input to the output. The class $\TCz$ consists of circuit families of size $\poly(n)$ that have constant depth and unbounded fan-in. The allowed gate types include AND, OR, NOT, and majority, and a majority gate enables counting within a single circuit layer. We call a circuit family \emph{shallow} if it lies in $\TCz$, or in $\ACz$ (under the constant-precision setting, detailed in Appendix~\ref{app:constprec}). The forward pass of a transformer follows precisely this structure (Figure~\ref{fig:pipeline}, top left), where its depth is fixed and does not depend on the input length $n$ (Figure~\ref{fig:grid}, left).

\begin{fact}[One forward pass of a transformer is shallow; \citealp{merrillsabharwal2023parallelism}]
\label{fact:onepass}
Take a transformer with a fixed (constant) depth, context length $\poly(n)$, and $O(\log n)$-bit numerical precision. For any given input, the result of a single forward pass can be computed by a circuit family in $\TCz$.
\end{fact}

In the following discussion, a transformer means a model as in Fact~\ref{fact:onepass}, and the theorems of Sections~\ref{sec:necessity} to~\ref{sec:locality} are stated for such models.

\paragraph{Pipelines and the ceiling.}

\begin{definition}[Pipeline]
\label{def:pipeline}
Fix a transformer $\pi$. A \emph{constant-pass pipeline} is a randomized procedure that maps an input $x$ to an answer in a \emph{constant} number of stages. Each stage processes $x$ and the outputs of the earlier stages, either by applying a shallow randomized map or by running one forward pass of $\pi$. The output of the last stage is the answer.
\end{definition}

Lemma~\ref{lem:pipeline} below shows that every pipeline is a shallow circuit (Figure~\ref{fig:pipeline}, bottom left), so the best accuracy that a shallow circuit of the same depth and size reaches on a task bounds every pipeline on that task. For constants $d$ and $s$, let $\mathcal{C}_n(d,s)$ be the set of randomized circuits of depth at most $d$ and size at most $n^s$ built from the gates of $\TCz$; a shallow randomized circuit family is a sequence $\{C_n\}$ with $C_n \in \mathcal{C}_n(d,s)$ for some fixed $d$ and $s$. For a task with target $y$ and input distribution $\mathcal{D}_n$, the \emph{ceiling} of the task at depth $d$ and size $n^s$ is
\[
\shal_n(d,s) \;:=\; \sup_{C \in \mathcal{C}_n(d,s)}\ \Pr_{x \sim \mathcal{D}_n,\ r}\big[\,C(x, r) = y(x)\,\big],
\]
where $r$ is the random string of the circuit. We write $\shal_n$ without the two constants when they are clear from the context, and Appendix~\ref{app:onepass} states how such a bound is read.

\begin{lemma}[Pipelines are shallow]
\label{lem:pipeline}
For every transformer $\pi$ and every constant-pass pipeline, there exist constants $d$ and $s$, depending only on $\pi$ and on the pipeline, and a family of shallow randomized circuits $C_n \in \mathcal{C}_n(d,s)$ whose output distribution on every input matches the pipeline’s answer distribution. Consequently, on every task and for every input distribution, the success of the pipeline is at most $\shal_n(d,s)$.
\end{lemma}

We include the detailed proof in Appendix~\ref{app:onepass}.

\paragraph{Three degrees of seriality.}
It is therefore important to identify which tasks have an upper bound that is strictly below one, indicating that the tasks cannot be solved with only a constant number of forward passes. Based on the fraction of inputs that pipelines can solve, we define three levels of \emph{seriality} as a property of the task (Figure~\ref{fig:grid}, right).

\begin{definition}[Degrees of seriality]
\label{def:serial}
For each input length $n$, a task consists of a target function $y$ and a distribution $\mathcal{D}_n$ over inputs, and the accuracy of any procedure on the task is the probability that it outputs $y(x)$ for $x \sim \mathcal{D}_n$. Let $b$ be the \emph{guessing baseline} of the task. The task is
\begin{itemize}
\item[(i)] \emph{inherently serial}, in the sense of \citet{li2024cot}, if no shallow circuit family computes $y$ on every input;
\item[(ii)] \emph{average-case serial} if its ceiling is bounded away from one, that is, if there is a constant $\eps > 0$ such that for every $d$ and $s$, $\shal_n(d,s) \le 1 - \eps$ for all sufficiently large $n$, so that every pipeline fails on at least an $\eps$ fraction of the inputs;
\item[(iii)] \emph{maximally serial} if its ceiling is the guessing baseline up to a negligible term, that is, if for every $d$ and $s$ and every constant $k$, $\shal_n(d,s) \le b + n^{-k}$ for all sufficiently large $n$, so that every pipeline does no better than guessing.
\end{itemize}
In (ii) and (iii) the threshold on $n$ may depend on $d$ and $s$.
\end{definition}

Condition (i) is a worst-case notion and does not depend on the input distribution. Conditions (ii) and (iii) are average-case notions, since the ceiling is an accuracy averaged over the input distribution of the task.

Our argument focuses on the second and third degrees, because on a task that is only inherently serial, a shallow circuit can succeed on all but a vanishing fraction of the inputs, so a model without a chain can reach accuracy close to one. Figure~\ref{fig:grid} (right) shows how much of the accuracy of a model has to come from its chain at each degree. The serial degrees are properties of the task, and we study what they imply for every model that tries to solve it.

Note that the degrees are relative to the precision setting, because shallow means $\TCz$ at log precision and $\ACz$ at constant precision. The main text uses the log-precision setting, following \citet{merrillsabharwal2023parallelism, merrillsabharwal2024cot}, and Appendix~\ref{app:parity} gives the constant-precision version of every result in the setting of \citet{li2024cot}.
\section{Necessity: the accuracy rests on the content of the chain}
\label{sec:necessity}

We first examine what Lemma~\ref{lem:pipeline} tells us about the necessity of CoT, by analyzing the performance drops upon replacing the chain after it has been produced. We argue that when the replacement does not depend on the content of the chain, the replacement and the answer head form a two-stage (constant-pass) pipeline, and the accuracy of every reasoning model thus drops to the ceiling of the task, no matter how it was trained.

\paragraph{Interventions and the load-bearing gap.}
Define an \emph{intervention} $\Phi$ on the CoT as a shallow randomized procedure that takes the input $x$ and the generated chain $c$ and returns a modified chain $\tilde c = \Phi(x, c)$ of length at most $\poly(n)$. The \emph{intervened model} $\pi^\Phi$ runs $\pi$ as a reasoning model, for $T$ forward passes to generate $c$, then replaces $c$ by $\tilde c$ and feeds it into the prompt to sample the answer from $\pi(\cdot \mid x, \tilde c)$, as illustrated by Figure~\ref{fig:interventions} in Appendix~\ref{app:proofs-necessity}. We study the \emph{load-bearing gap} $\gap_\Phi(\pi) := \suc_n(\pi) - \suc_n(\pi^\Phi)$, the accuracy that the model loses when its chain is replaced. An intervention is \emph{erasing} if it does not read the chain, that is, $\Phi(x, c) = \varphi(x)$ for a shallow randomized map $\varphi$ (Figure~\ref{fig:pipeline}, right). Interventions that read the chain, such as truncation and shuffling, are treated in Section~\ref{sec:locality}.

\begin{theorem}[Necessity]
\label{thm:necessity}
Let $\pi$ be any transformer and let $\Phi$ be any erasing intervention. There are constants $d$ and $s$, depending only on $\pi$ and $\Phi$, such that on every task and for every input distribution, $\suc_n(\pi^\Phi) \leq \shal_n(d,s)$, and hence $\gap_\Phi(\pi) \geq \suc_n(\pi) - \shal_n(d,s)$.
\end{theorem}

On a maximally serial task, the accuracy of a model rests entirely on the content of its chain. Training can only change what the model writes in its chain. It can never produce a model that keeps its accuracy without reasoning, because the theorem constrains what any transformer can compute at test time. 

Appendix~\ref{sec:discussion} discusses what our theorems imply about current reasoning models, including one-pass models such as Jev, latent reasoning, diffusion decoding, chain compression, and process reward models.

\section{Depth: no shallow computation can write the chain}
\label{sec:depth}

In Section~\ref{sec:necessity}, Theorem~\ref{thm:necessity} bounds the accuracy of a model after its chain has been replaced. In this section, we ask what happens if a shallow map produces chains close to the model's own. We argue that the accuracy changes by at most the distance between the two chain distributions, and the swapped-in version is capped by the ceiling, so the chain of a successful reasoning model cannot be replaced by a shallow computation.

Define a \emph{shallow imitator} of $\pi$ to be a shallow randomized family $\varphi$ that maps an input $x$ to a chain, and its \emph{imitation error} $\eta_n := \E_x\, \TV\big(\varphi(x), \pi_c(\cdot \mid x)\big)$ is the average distance between its chains and the chains $\pi_c(\cdot \mid x)$ that the model generates. Here $\TV(P, Q) := \max_A |P(A) - Q(A)|$ is the total variation distance, the largest difference in probability that $P$ and $Q$ assign to any set of chains; it is $0$ when the two distributions are identical and $1$ when no chain can be produced by both.

\begin{theorem}[Depth]
\label{thm:depth}
Let $\pi$ be any transformer and let $\varphi$ be any shallow imitator of $\pi$ with imitation error $\eta_n$. There are constants $d$ and $s$, depending only on $\pi$ and $\varphi$, such that on every task and for every input distribution, $\suc_n(\pi) - \shal_n(d,s) \leq \eta_n$.
\end{theorem}

Because the total variation distance between a chain produced by shallow generation and the original chain from the successful reasoner is at least as large as the performance gap between that reasoner and a pipeline, it follows that the output chains of an accurate model are not even approximately close to the set of chains that any shallow computation can produce. Shallow computation can only generate chains where each token depends on previous tokens through a constant-bounded number of rounds. This further indicates that a successful model has to produce chains whose tokens are built on one another over a number of rounds that increases with the input.
\section{Locality: all of the multi-step work happens in the chain}
\label{sec:locality}

Finally, we apply Lemma~\ref{lem:pipeline} to the answer head. It produces the answer with one forward pass after the chain ends, so it is shallow and cannot contribute any serial work. We therefore conclude that by the time the chain is generated, the answer must already be computable by a shallow circuit from the input and the finished chain. 

\begin{definition}[Shallow decodability]
\label{def:decodability}
A \emph{decoder} is a shallow randomized circuit that predicts the target $y(x)$ from the input $x$ and a chain. For a transformer $\pi$, an intervention $\Phi$, and constants $d$ and $s$, the \emph{shallow decodability} $\decode_\Phi(\pi; d,s)$ is the highest accuracy that any decoder in $\mathcal{C}_n(d,s)$ reaches on $x$ and the intervened chain $\Phi(x, c)$, where $c$ is the chain that the model generates; $\decode(\pi; d,s)$ denotes the same quantity when no intervention is applied, which is the highest accuracy that any such decoder reaches from $x$ and the chain the model generates. As for the ceiling, we write $\decode_\Phi(\pi)$ and $\decode(\pi)$ when the two constants are clear from the context.
\end{definition}

\begin{theorem}[Locality]
\label{thm:locality}
For every transformer $\pi$ there are constants $d$ and $s$, depending only on $\pi$, such that for every intervention $\Phi$, every task, and every input distribution, $\suc_n(\pi^\Phi) \leq \decode_\Phi(\pi; d,s)$; in particular $\suc_n(\pi) \leq \decode(\pi; d,s)$. Moreover, when the intervention $\Phi$ is erasing (Section~\ref{sec:necessity}), for every $d$ and $s$ there are constants $d'$ and $s'$, depending only on $d$, $s$, and $\Phi$, such that $\decode_\Phi(\pi; d,s) \leq \shal_n(d',s')$.
\end{theorem}

The theorem tells us when a reasoning model has its answer. The answer is committed by the time the chain ends, in the sense that a shallow computation with no reasoning ability of its own could read it from the input and the finished chain.

\section{Experiments}
\label{sec:experiments}

We conduct two sets of experiments: (1) We train small transformers on word problems of finite groups, whose ceilings are known, so every intervened accuracy in Sections~\ref{sec:necessity} through~\ref{sec:locality} has an exact predicted value. (2) We apply the interventions of Sections~\ref{sec:necessity} and~\ref{sec:locality} to open-source reasoning models, first on the word problem of $A_5$, and then on mathematical benchmarks, whose ceiling is unknown. The full setup is given in Appendix~\ref{app:exp}, and additional results are given in Appendix~\ref{app:more-exp}.

\subsection{Tasks with a known ceiling: small transformers trained from scratch on word problems}
\label{sec:exp-toy}
For a finite group $G$, the word problem $\WP_G$ at length $n$ asks for the product $y(x) = g_1 g_2 \cdots g_n$ of an input $x = (g_1, \dots, g_n)$ drawn uniformly from $G^n$ unless stated otherwise, with each $g_i$ encoded as one token, so $b = 1/|G|$, and a chain that writes the running product once per element solves it exactly. Its ceiling is known, and it reaches each of the three degrees when defined on different groups \citep{liu2023shortcuts, merrill2024illusion}, as detailed in Appendix~\ref{app:degrees}: the maximally serial case is $\WP_n$, the word problem of $A_5$ on uniform inputs; the average-case serial case is the word problem of $S_5$; and the non-serial cases, whose ceiling is 1, are the solvable groups $\mathbb{Z}_{60}$ and $A_4 \times \mathbb{Z}_5$ (Table~\ref{tab:prices} in Appendix~\ref{app:proofs-necessity}). In the maximally serial case the ceiling is $\tfrac{1}{60} + n^{-k}$ (Lemma~\ref{lem:avgcase}).

\begin{figure}[t]
\centering
{\offinterlineskip
\includegraphics[width=\textwidth, trim=0 254.1bp 0 0.0bp, clip]{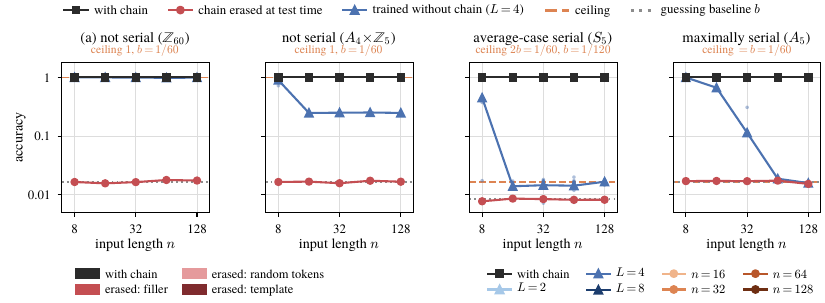}\par
\includegraphics[width=\textwidth, trim=0 139.1bp 0 157.0bp, clip]{figures/section7/fig_small.pdf}\par
\includegraphics[width=\textwidth, trim=0 122.6bp 0 264.5bp, clip]{figures/section7/fig_small.pdf}\par
\includegraphics[width=\textwidth, trim=0 13.1bp 0 287.0bp, clip]{figures/section7/fig_small.pdf}\par
}
\caption{Small transformers on word problems ($L = 4$ unless stated; $n = 8$ to $128$). \emph{(a)} Necessity, from the least to the most serial group: models trained with the chain solve every length (black); erasing their chain at test time drops them to the baseline $b$ on every group (Theorem~\ref{thm:necessity}). \emph{(b)} The same drop, for each of the three erasing interventions, after reinforcement learning with GRPO, RLOO, or PPO on $A_5$ at $n = 32$. \emph{(c)} Depth: without a chain, accuracy stays high only while one pass can track the input, and the collapse length grows with $L$. \emph{(d)} The tracked prefix $t^\ast$ levels off at a value set by $L$. \emph{(e)} Models trained to write a chain of $T$ steps: only $T = n$ solves the task.  \emph{(f)} Locality on $A_5$ for three chain formats. After a token shuffle the answer head reads only the bag of chain tokens, so its accuracy (green) cannot exceed that of the best bag decoder trained on the bag ($\triangledown$) (Theorem~\ref{thm:locality}, Lemma~\ref{lem:bag}). The bag reveals the answer only when the chain repeats it $2n+1$ times (Proposition~\ref{prop:dissoc}). }
\label{fig:small}
\end{figure}

\textbf{On the maximally serial case ($\WP_n$ of $A_5$), erasing the chain drops the performance from 1 to the predicted ceiling.} The chain-trained models reach accuracy $1.000$ on $A_5$ at every length up to $n = 128$ (Figure~\ref{fig:small}a). When their chain is replaced by filler, by random tokens, or by a template, the accuracy drops to between $0.0148$ and $0.0188$ over all lengths against the baseline $1/60 = 0.0167$; the quantitative ceiling predicted by Theorem~\ref{thm:necessity} is exactly met. Therefore, the whole accuracy above chance, $1 - b$, rests on the content of the chain.

\textbf{After reinforcement learning that rewards only the final answer, erasing the chain still drops the accuracy to the ceiling.} Models trained on $A_5$ at $n = 32$ with a chain and then with GRPO \citep{shao2024deepseekmath}, RLOO \citep{ahmadian2024back}, or PPO \citep{schulman2017ppo} end at $0.709$ to $0.804$ with the chain and at $0.014$ to $0.019$ under every erasing intervention (Figure~\ref{fig:small}b).

\textbf{What survives without a chain is the part of the answer that counting can compute.} At $n = 128$ the chainless models keep exactly the accuracy of computing the commutative part of the product and guessing the rest (Proposition~\ref{prop:quotient};
Figure~\ref{fig:small}a, blue): all of it on $\mathbb{Z}_{60}$; $1/4$ on $A_4 \times \mathbb{Z}_5$, although its ceiling is $1$, because training does not find the shallow circuit for the rest (Appendix~\ref{app:more-exp}).

\textbf{Without a chain, accuracy collapses beyond an input length that grows with the depth, even with supervision at every step.} We follow \citet{li2024cot} to train with dense supervision, which provides the model with the running product at every position, so the collapse is not due to insufficient training signal. On $A_5$ the collapse comes at $n = 16$ with $L = 2$ and at $n = 64$ with $L = 4$, while $L = 8$ still holds $0.886$ at $n = 64$ (Figure~\ref{fig:small}c). By Fact~\ref{fact:onepass}, a pass of fixed depth (pipeline) performs a bounded number of sequential steps, and for a sufficiently large $n$, no training makes a single pass sufficient.

\textbf{One forward pass can perform only a bounded number of sequential steps.} Figure~\ref{fig:small}d measures this maximum number of steps a forward pass can perform by tracking the prefix $t^\ast$, the largest position up to which a chainless model still computes the running product correctly. In the maximally serial case ($\WP_n$), $t^\ast$ stays at about $8$ at $L = 2$ for inputs of all lengths, stays at about $30$ at $L = 4$, and reaches about $110$ at $L = 8$. The only way to extend the bound is a chain, since each generated token adds one more pass (Figure~\ref{fig:circuits}).

\textbf{A chain helps only when its length grows with the input.} We train models whose chain writes only every $k$-th running product, so the chain has $T = n/k$ steps and each step has to multiply $k$ elements in one pass. Every $T < n$ stays at the baseline and only $T = n$ reaches $1.000$ (Figure~\ref{fig:small}e).

\subsection{A task with a known ceiling at scale: open-weight reasoning models on the word problem}
\label{sec:exp-wp}
We pose the word problem of $A_5$ to Qwen3 models \citep{yang2025qwen3} with 4B to 32B parameters in natural language: each element of $A_5$ is given the name of an orientation of an object; the prompt states the starting orientation, the rotations to apply in order, and for each rotation the orientation it produces from every orientation; the model has to name the final orientation, one of $60$ ($b = 1/60$). For $200$ words at each length $n$ from $4$ to $16$ we run two protocols: we let the model think for at most $T$ tokens and then force the answer as in Section~\ref{sec:exp-real}, for $T$ from $0$ to $2{,}048$; or we let it finish its chain and replace the chain by filler of the same length before forcing the answer (Appendix~\ref{app:exp}).
\begin{figure}[t]
\centering
{\offinterlineskip
\includegraphics[width=\textwidth, trim=0 112.5bp 0 0.0bp, clip]{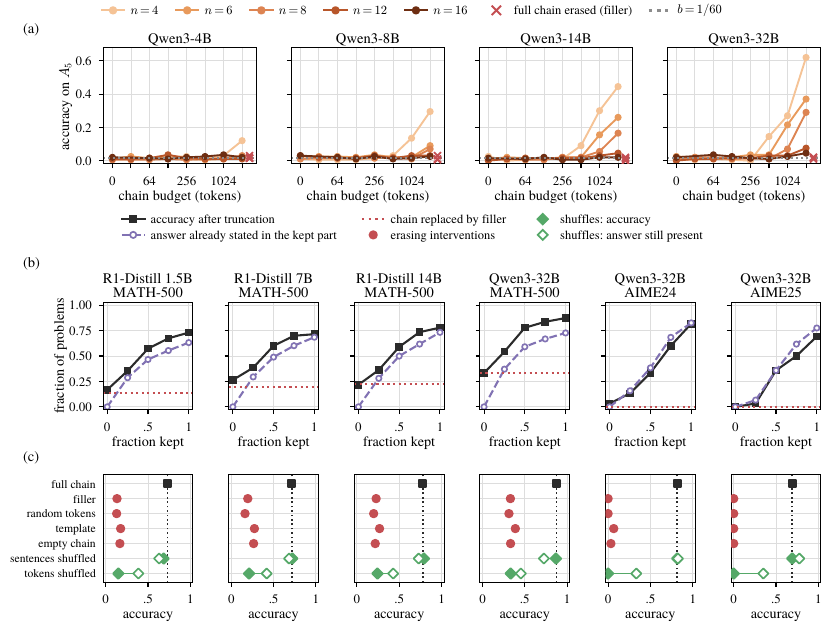}\par
\includegraphics[width=\textwidth, trim=0 4.4bp 0 308.8bp, clip]{figures/section7/fig_open.pdf}\par
}
\caption{Open-weight reasoning models. \emph{(a)} Qwen3 models on $A_5$ (maximally serial): accuracy after a chain of at most the given number of tokens. \emph{(b)} Mathematical benchmarks: accuracy after truncating the chain to a prefix (solid) against the fraction of problems whose kept prefix already states the answer (dashed). \emph{(c)} Accuracy after each intervention: the four erasing interventions (red) drop the accuracy, a sentence shuffle keeps the accuracy of the full chain, and a token shuffle falls to the level of erasing although some answer is still present in the shuffled chains (hollow diamonds). \emph{(d)} The four untrained models (filled) and $19$ checkpoints of Qwen3 models trained by GRPO with a correct or a random reward (hollow), on MATH-500: accuracy after intervention against the accuracy with the full chain (Table~\ref{tab:app-checkpoints}).}
\label{fig:open}
\end{figure}

\textbf{Erasing returns the open models to the baseline, and a chain helps only from $512$ tokens on.} With its full chain Qwen3-32B reaches $0.62$ at $n = 4$ and $0.29$ at $n = 8$; with filler, or with no chain at all, every model scores at most $0.030$ at every $n$, and equally on $\mathbb{Z}_{60}$ and $A_4 \times \mathbb{Z}_5$ (Figure~\ref{fig:open}a, Appendix~\ref{app:more-exp}). With chains of at most $2{,}048$ tokens, the accuracy falls with $n$, from $0.62$ at $n = 4$ to $0.075$ at $n = 12$ for Qwen3-32B, as Theorem~\ref{thm:depth} states; beyond $n = 8$ every model scores at most $0.075$ even with its full chain.

\subsection{Tasks with an unknown ceiling: open-weight reasoning models on mathematical benchmarks}
\label{sec:exp-real}
To test the implications of Section~\ref{sec:necessity} through Section~\ref{sec:locality}, we apply the interventions to DeepSeek-R1-Distill-Qwen models \citep{guo2025deepseekr1} with 1.5B, 7B, and 14B parameters and to Qwen3-32B, on AIME24, AIME25, and MATH-500. Each intervention edits the content of the thinking block, and the model is then forced to state its answer immediately without reopening the block, so everything after the intervened chain is the answer head defined in Section~\ref{sec:prelim}. The content-dependent interventions of Section~\ref{sec:locality} are truncation to a prefix (keeping $0\%$ of the chain gives an empty chain, which is erasing), sentence shuffle, and token shuffle.

\textbf{Erasing the chain removes 0.52 to 0.82 accuracy, and what survives is the part that the model can solve without its chain.} Replacing the chain by filler costs the four models $0.52$ to $0.60$ accuracy on MATH-500 and costs Qwen3-32B $0.82$ on AIME24 and $0.69$ on AIME25; the four erasing interventions agree within $0.11$ (Figure~\ref{fig:open}c). From Section~\ref{sec:necessity}, the surviving accuracy is the part of the benchmark that the model solves without any serial computation. On MATH-500 this floor is $0.13$ to $0.33$ and is highest for the largest model. The part carried by the chain, $0.52$ to $0.60$, stays roughly constant across sizes, indicating that a larger model solves more problems in one pass, but the number of problems that need the chain does not go down.

\textbf{The model answers correctly in step with how often the kept part of the chain already states the answer.} For Qwen3-32B on AIME24, keeping $25$, $50$, and $75$ percent of the chain gives accuracy $0.133$, $0.333$, and $0.600$, while the fraction of problems whose answer is already \emph{present} in the kept part is $0.158$, $0.383$, and $0.683$ (Figure~\ref{fig:open}b). By Theorem~\ref{thm:locality}, wherever accuracy survives truncation the answer was already readable in the kept prefix. The presence curve shows that the answer enters the chain progressively rather than at the start. 

\textbf{Shuffling sentences keeps the answer readable while shuffling tokens does not, even though the answer is still there.} Shuffling the sentences of the chain changes accuracy by at most $0.045$ on every model and benchmark (Figure~\ref{fig:open}c). By Theorem~\ref{thm:locality} and Lemma~\ref{lem:bag} in Appendix~\ref{app:proofs-locality}, this certifies that the answer is shallowly readable from the bag of sentences: the model only needs to select the sentence that states the answer, which is a shallow operation. Shuffling the tokens drops accuracy to the level of an empty chain, $0.000$ on AIME and $0.333$ on MATH-500, although the answer string is still present in $33\%$ to $45\%$ of the token-shuffled chains. Recovering the answer from a bag of tokens would mean reordering them, which is serial work that a shallow answer head cannot do. The small models show the same bound (Figure~\ref{fig:small}f).

\section{Conclusion}
\label{sec:conclusion}

In this work, we have shown that a transformer computes each token with one shallow pass, so any procedure that runs it a constant number of times is a shallow circuit. The best accuracy a shallow circuit reaches on a task, the ceiling of the task, bounds the accuracy of every such procedure. We further introduced three theorems that hold for every transformer, no matter how it is trained. Everything a model achieves above the ceiling rests on the content of its chain (Necessity); such a chain cannot be written by a shallow computation, not even approximately (Depth); and the answer is committed by the time the chain ends and is read off by one more shallow pass (Locality). We ran extensive experiments with small transformer models on word problems with known ceilings. We also evaluated open-weight reasoning models on the same word problems, as well as on mathematical benchmarks. Both families of models were tested with and without reinforcement learning. All the results meet what our theorems predict exactly. Our theorems also apply to current systems from one-pass models to diffusion decoders and process rewards (Appendix~\ref{sec:discussion}). This is a step toward deciding, from the task rather than the model, when a chain of thought is necessary and when a single pass suffices. We hope our work can offer insights for the next harness system that integrates models with varying reasoning capacities for different real-world tasks.

\bibliography{iclr2027_conference}

\begin{thebibliography}{39}
\providecommand{\natexlab}[1]{#1}
\providecommand{\url}[1]{\texttt{#1}}
\expandafter\ifx\csname urlstyle\endcsname\relax
  \providecommand{\doi}[1]{doi: #1}\else
  \providecommand{\doi}{doi: \begingroup \urlstyle{rm}\Url}\fi

\bibitem[Adleman(1978)]{adleman1978}
Leonard~M. Adleman.
\newblock Two theorems on random polynomial time.
\newblock In \emph{19th Annual Symposium on Foundations of Computer Science
  (FOCS)}, pp.\  75--83, 1978.

\bibitem[Ahmadian et~al.(2024)Ahmadian, Cremer, Gall{\'e}, Fadaee, Kreutzer,
  Pietquin, {\"U}st{\"u}n, and Hooker]{ahmadian2024back}
Arash Ahmadian, Chris Cremer, Matthias Gall{\'e}, Marzieh Fadaee, Julia
  Kreutzer, Olivier Pietquin, Ahmet {\"U}st{\"u}n, and Sara Hooker.
\newblock Back to basics: Revisiting {REINFORCE}-style optimization for
  learning from human feedback in {LLMs}.
\newblock In \emph{Proceedings of the 62nd Annual Meeting of the Association
  for Computational Linguistics (ACL)}, 2024.

\bibitem[Almeida(2026)]{typesafe2026jev}
Diogo Almeida.
\newblock Introducing {S}ystem {O}ne models \& {J}ev.
\newblock TypeSafe AI blog,
  \url{https://typesafe.ai/blog/introducing-system-one-models-and-jev},
  September 2026.
\newblock Accessed September 22, 2026.

\bibitem[Amiri et~al.(2025)Amiri, Huang, Rofin, and Hahn]{amiri2025lower}
Alireza Amiri, Xinting Huang, Mark Rofin, and Michael Hahn.
\newblock Lower bounds for chain-of-thought reasoning in hard-attention
  transformers.
\newblock \emph{arXiv preprint arXiv:2502.02393}, 2025.

\bibitem[Barrington(1989)]{barrington1989}
David A.~Mix Barrington.
\newblock Bounded-width polynomial-size branching programs recognize exactly
  those languages in {NC}$^1$.
\newblock \emph{Journal of Computer and System Sciences}, 38\penalty0
  (1):\penalty0 150--164, 1989.

\bibitem[Barrington \& Th{\'e}rien(1988)Barrington and
  Th{\'e}rien]{barringtontherien1988}
David A.~Mix Barrington and Denis Th{\'e}rien.
\newblock Finite monoids and the fine structure of {NC}$^1$.
\newblock \emph{Journal of the ACM}, 35\penalty0 (4):\penalty0 941--952, 1988.

\bibitem[Chen et~al.(2024)Chen, Xu, Liang, He, Pang, Yu, Song, Liu, Zhou,
  Zhang, Wang, Tu, Mi, and Yu]{chen2024overthinking}
Xingyu Chen, Jiahao Xu, Tian Liang, Zhiwei He, Jianhui Pang, Dian Yu, Linfeng
  Song, Qiuzhi Liu, Mengfei Zhou, Zhuosheng Zhang, Rui Wang, Zhaopeng Tu,
  Haitao Mi, and Dong Yu.
\newblock Do {NOT} think that much for 2+3=? {On} the overthinking of o1-like
  {LLMs}.
\newblock \emph{arXiv preprint arXiv:2412.21187}, 2024.

\bibitem[Chen et~al.(2025)Chen, Benton, Radhakrishnan, Uesato, Denison,
  Schulman, Somani, Hase, Wagner, Roger, Mikulik, Bowman, Leike, Kaplan, and
  Perez]{chen2025reasoning}
Yanda Chen, Joe Benton, Ansh Radhakrishnan, Jonathan Uesato, Carson Denison,
  John Schulman, Arushi Somani, Peter Hase, Misha Wagner, Fabien Roger, Vlad
  Mikulik, Samuel~R. Bowman, Jan Leike, Jared Kaplan, and Ethan Perez.
\newblock Reasoning models don't always say what they think.
\newblock \emph{arXiv preprint arXiv:2505.05410}, 2025.

\bibitem[Deng et~al.(2023)Deng, Prasad, Fernandez, Smolensky, Chaudhary, and
  Shieber]{deng2023implicit}
Yuntian Deng, Kiran Prasad, Roland Fernandez, Paul Smolensky, Vishrav
  Chaudhary, and Stuart Shieber.
\newblock Implicit chain of thought reasoning via knowledge distillation.
\newblock \emph{arXiv preprint arXiv:2311.01460}, 2023.

\bibitem[Deng et~al.(2024)Deng, Choi, and Shieber]{deng2024implicit}
Yuntian Deng, Yejin Choi, and Stuart Shieber.
\newblock From explicit {CoT} to implicit {CoT}: Learning to internalize {CoT}
  step by step.
\newblock \emph{arXiv preprint arXiv:2405.14838}, 2024.

\bibitem[Geiping et~al.(2025)Geiping, McLeish, Jain, Kirchenbauer, Singh,
  Bartoldson, Kailkhura, Bhatele, and Goldstein]{geiping2025recurrent}
Jonas Geiping, Sean McLeish, Neel Jain, John Kirchenbauer, Siddharth Singh,
  Brian~R. Bartoldson, Bhavya Kailkhura, Abhinav Bhatele, and Tom Goldstein.
\newblock Scaling up test-time compute with latent reasoning: A recurrent depth
  approach.
\newblock In \emph{Advances in Neural Information Processing Systems
  (NeurIPS)}, 2025.
\newblock arXiv:2502.05171.

\bibitem[Goyal et~al.(2024)Goyal, Ji, Rawat, Menon, Kumar, and
  Nagarajan]{goyal2024pause}
Sachin Goyal, Ziwei Ji, Ankit~Singh Rawat, Aditya~Krishna Menon, Sanjiv Kumar,
  and Vaishnavh Nagarajan.
\newblock Think before you speak: Training language models with pause tokens.
\newblock In \emph{International Conference on Learning Representations
  (ICLR)}, 2024.
\newblock arXiv:2310.02226.

\bibitem[Gu et~al.(2018)Gu, Bradbury, Xiong, Li, and
  Socher]{gu2018nonautoregressive}
Jiatao Gu, James Bradbury, Caiming Xiong, Victor O.~K. Li, and Richard Socher.
\newblock Non-autoregressive neural machine translation.
\newblock In \emph{International Conference on Learning Representations
  (ICLR)}, 2018.
\newblock arXiv:1711.02281.

\bibitem[Guo et~al.(2025)Guo, Yang, Zhang, Song, Zhang, Xu, Zhu, Ma, Wang, Bi,
  et~al.]{guo2025deepseekr1}
Daya Guo, Dejian Yang, Haowei Zhang, Junxiao Song, Ruoyu Zhang, Runxin Xu,
  Qihao Zhu, Shirong Ma, Peiyi Wang, Xiao Bi, et~al.
\newblock {DeepSeek-R1}: Incentivizing reasoning capability in {LLMs} via
  reinforcement learning.
\newblock \emph{arXiv preprint arXiv:2501.12948}, 2025.

\bibitem[Hao et~al.(2025)Hao, Sukhbaatar, Su, Li, Hu, Weston, and
  Tian]{hao2024coconut}
Shibo Hao, Sainbayar Sukhbaatar, DiJia Su, Xian Li, Zhiting Hu, Jason Weston,
  and Yuandong Tian.
\newblock Training large language models to reason in a continuous latent
  space.
\newblock In \emph{Conference on Language Modeling (COLM)}, 2025.
\newblock arXiv:2412.06769.

\bibitem[H{\aa}stad(1987)]{hastad1987}
Johan H{\aa}stad.
\newblock \emph{Computational Limitations of Small-Depth Circuits}.
\newblock MIT Press, 1987.

\bibitem[H{\aa}stad(2014)]{hastad2014}
Johan H{\aa}stad.
\newblock On the correlation of parity and small-depth circuits.
\newblock \emph{SIAM Journal on Computing}, 43\penalty0 (5):\penalty0
  1699--1708, 2014.

\bibitem[Impagliazzo et~al.(2012)Impagliazzo, Matthews, and Paturi]{imp2012}
Russell Impagliazzo, William Matthews, and Ramamohan Paturi.
\newblock A satisfiability algorithm for {AC}$^0$.
\newblock In \emph{Proceedings of the Twenty-Third Annual ACM-SIAM Symposium on
  Discrete Algorithms (SODA)}, pp.\  961--972, 2012.

\bibitem[{Inception Labs} et~al.(2025){Inception Labs}, Khanna, Kharbanda, Li,
  Varma, Wang, Birnbaum, Luo, Miraoui, Palrecha, Ermon, Grover, and
  Kuleshov]{inception2025mercury}
{Inception Labs}, Samar Khanna, Siddhant Kharbanda, Shufan Li, Harshit Varma,
  Eric Wang, Sawyer Birnbaum, Ziyang Luo, Yanis Miraoui, Akash Palrecha,
  Stefano Ermon, Aditya Grover, and Volodymyr Kuleshov.
\newblock Mercury: Ultra-fast language models based on diffusion.
\newblock \emph{arXiv preprint arXiv:2506.17298}, 2025.

\bibitem[Kahneman(2011)]{kahneman2011thinking}
Daniel Kahneman.
\newblock \emph{Thinking, Fast and Slow}.
\newblock Farrar, Straus and Giroux, 2011.

\bibitem[Lanham et~al.(2023)Lanham, Chen, Radhakrishnan, Steiner, Denison,
  Hernandez, Li, Durmus, Hubinger, Kernion, Luko{\v s}i{\=u}t{\.e}, Nguyen,
  Cheng, Joseph, Schiefer, Rausch, Larson, McCandlish, Kundu, Kadavath, Yang,
  Henighan, Maxwell, Telleen-Lawton, Hume, Hatfield-Dodds, Kaplan, Brauner,
  Bowman, and Perez]{lanham2023}
Tamera Lanham, Anna Chen, Ansh Radhakrishnan, Benoit Steiner, Carson Denison,
  Danny Hernandez, Dustin Li, Esin Durmus, Evan Hubinger, Jackson Kernion,
  Kamil{\.e} Luko{\v s}i{\=u}t{\.e}, Karina Nguyen, Newton Cheng, Nicholas
  Joseph, Nicholas Schiefer, Oliver Rausch, Robin Larson, Sam McCandlish,
  Sandipan Kundu, Saurav Kadavath, Shannon Yang, Thomas Henighan, Timothy
  Maxwell, Timothy Telleen-Lawton, Tristan Hume, Zac Hatfield-Dodds, Jared
  Kaplan, Jan Brauner, Samuel~R. Bowman, and Ethan Perez.
\newblock Measuring faithfulness in chain-of-thought reasoning.
\newblock \emph{arXiv preprint arXiv:2307.13702}, 2023.

\bibitem[Li et~al.(2024)Li, Liu, Zhou, and Ma]{li2024cot}
Zhiyuan Li, Hong Liu, Denny Zhou, and Tengyu Ma.
\newblock Chain of thought empowers transformers to solve inherently serial
  problems.
\newblock In \emph{International Conference on Learning Representations
  (ICLR)}, 2024.
\newblock arXiv:2402.12875.

\bibitem[Lightman et~al.(2024)Lightman, Kosaraju, Burda, Edwards, Baker, Lee,
  Leike, Schulman, Sutskever, and Cobbe]{lightman2023verify}
Hunter Lightman, Vineet Kosaraju, Yuri Burda, Harrison Edwards, Bowen Baker,
  Teddy Lee, Jan Leike, John Schulman, Ilya Sutskever, and Karl Cobbe.
\newblock Let's verify step by step.
\newblock In \emph{International Conference on Learning Representations
  (ICLR)}, 2024.
\newblock arXiv:2305.20050.

\bibitem[Liu et~al.(2023)Liu, Ash, Goel, Krishnamurthy, and
  Zhang]{liu2023shortcuts}
Bingbin Liu, Jordan~T. Ash, Surbhi Goel, Akshay Krishnamurthy, and Cyril Zhang.
\newblock Transformers learn shortcuts to automata.
\newblock In \emph{International Conference on Learning Representations
  (ICLR)}, 2023.
\newblock arXiv:2210.10749.

\bibitem[Merrill \& Sabharwal(2023)Merrill and
  Sabharwal]{merrillsabharwal2023parallelism}
William Merrill and Ashish Sabharwal.
\newblock The parallelism tradeoff: Limitations of log-precision transformers.
\newblock \emph{Transactions of the Association for Computational Linguistics},
  11:\penalty0 531--545, 2023.
\newblock arXiv:2207.00729.

\bibitem[Merrill \& Sabharwal(2024)Merrill and
  Sabharwal]{merrillsabharwal2024cot}
William Merrill and Ashish Sabharwal.
\newblock The expressive power of transformers with chain of thought.
\newblock In \emph{International Conference on Learning Representations
  (ICLR)}, 2024.
\newblock arXiv:2310.07923.

\bibitem[Merrill \& Sabharwal(2025)Merrill and
  Sabharwal]{merrillsabharwal2025depth}
William Merrill and Ashish Sabharwal.
\newblock A little depth goes a long way: The expressive power of log-depth
  transformers.
\newblock In \emph{Advances in Neural Information Processing Systems
  (NeurIPS)}, 2025.
\newblock arXiv:2503.03961.

\bibitem[Merrill et~al.(2024)Merrill, Petty, and
  Sabharwal]{merrill2024illusion}
William Merrill, Jackson Petty, and Ashish Sabharwal.
\newblock The illusion of state in state-space models.
\newblock In \emph{International Conference on Machine Learning (ICML)}, 2024.

\bibitem[Miles \& Viola(2013)Miles and Viola]{milesviola2013}
Eric Miles and Emanuele Viola.
\newblock Shielding circuits with groups.
\newblock In \emph{45th ACM Symposium on Theory of Computing (STOC)}, pp.\
  251--260, 2013.
\newblock ECCC TR13-003; IACR ePrint 2013/001.

\bibitem[Nie et~al.(2025)Nie, Zhu, You, Zhang, Ou, Hu, Zhou, Lin, Wen, and
  Li]{nie2025llada}
Shen Nie, Fengqi Zhu, Zebin You, Xiaolu Zhang, Jingyang Ou, Jun Hu, Jun Zhou,
  Yankai Lin, Ji-Rong Wen, and Chongxuan Li.
\newblock Large language diffusion models.
\newblock In \emph{Advances in Neural Information Processing Systems
  (NeurIPS)}, 2025.
\newblock arXiv:2502.09992.

\bibitem[Pfau et~al.(2024)Pfau, Merrill, and Bowman]{pfau2024}
Jacob Pfau, William Merrill, and Samuel~R. Bowman.
\newblock Let's think dot by dot: Hidden computation in transformer language
  models.
\newblock \emph{arXiv preprint arXiv:2404.15758}, 2024.

\bibitem[Saunshi et~al.(2025)Saunshi, Dikkala, Li, Kumar, and
  Reddi]{saunshi2025looped}
Nikunj Saunshi, Nishanth Dikkala, Zhiyuan Li, Sanjiv Kumar, and Sashank~J.
  Reddi.
\newblock Reasoning with latent thoughts: On the power of looped transformers.
\newblock In \emph{International Conference on Learning Representations
  (ICLR)}, 2025.
\newblock arXiv:2502.17416.

\bibitem[Schulman et~al.(2017)Schulman, Wolski, Dhariwal, Radford, and
  Klimov]{schulman2017ppo}
John Schulman, Filip Wolski, Prafulla Dhariwal, Alec Radford, and Oleg Klimov.
\newblock Proximal policy optimization algorithms.
\newblock \emph{arXiv preprint arXiv:1707.06347}, 2017.

\bibitem[Shao et~al.(2025)Shao, Li, Xin, Geng, Wang, Oh, Du, Lambert, Min,
  Krishna, Tsvetkov, Hajishirzi, Koh, and Zettlemoyer]{shao2025spurious}
Rulin Shao, Shuyue~Stella Li, Rui Xin, Scott Geng, Yiping Wang, Sewoong Oh,
  Simon~Shaolei Du, Nathan Lambert, Sewon Min, Ranjay Krishna, Yulia Tsvetkov,
  Hannaneh Hajishirzi, Pang~Wei Koh, and Luke Zettlemoyer.
\newblock Spurious rewards: Rethinking training signals in {RLVR}.
\newblock \emph{arXiv preprint arXiv:2506.10947}, 2025.

\bibitem[Shao et~al.(2024)Shao, Wang, Zhu, Xu, Song, Bi, Zhang, Zhang, Li, Wu,
  and Guo]{shao2024deepseekmath}
Zhihong Shao, Peiyi Wang, Qihao Zhu, Runxin Xu, Junxiao Song, Xiao Bi, Haowei
  Zhang, Mingchuan Zhang, Y.~K. Li, Y.~Wu, and Daya Guo.
\newblock {DeepSeekMath}: Pushing the limits of mathematical reasoning in open
  language models.
\newblock \emph{arXiv preprint arXiv:2402.03300}, 2024.

\bibitem[Sheng et~al.(2024)Sheng, Zhang, Ye, Wu, Zhang, Zhang, Peng, Lin, and
  Wu]{sheng2024hybridflow}
Guangming Sheng, Chi Zhang, Zilingfeng Ye, Xibin Wu, Wang Zhang, Ru~Zhang,
  Yanghua Peng, Haibin Lin, and Chuan Wu.
\newblock {HybridFlow}: A flexible and efficient {RLHF} framework.
\newblock \emph{arXiv preprint arXiv:2409.19256}, 2024.

\bibitem[Turpin et~al.(2023)Turpin, Michael, Perez, and Bowman]{turpin2023}
Miles Turpin, Julian Michael, Ethan Perez, and Samuel~R. Bowman.
\newblock Language models don't always say what they think: Unfaithful
  explanations in chain-of-thought prompting.
\newblock In \emph{Advances in Neural Information Processing Systems
  (NeurIPS)}, 2023.
\newblock arXiv:2305.04388.

\bibitem[Xu et~al.(2025)Xu, Xie, Zhao, and He]{xu2025cod}
Silei Xu, Wenhao Xie, Lingxiao Zhao, and Pengcheng He.
\newblock Chain of draft: Thinking faster by writing less.
\newblock \emph{arXiv preprint arXiv:2502.18600}, 2025.

\bibitem[Yang et~al.(2025)Yang, Li, Yang, Zhang, Hui, Zheng, Yu, Gao, Huang,
  Lv, et~al.]{yang2025qwen3}
An~Yang, Anfeng Li, Baosong Yang, Beichen Zhang, Binyuan Hui, Bo~Zheng, Bowen
  Yu, Chang Gao, Chengen Huang, Chenxu Lv, et~al.
\newblock {Qwen3} technical report.
\newblock \emph{arXiv preprint arXiv:2505.09388}, 2025.

\end{thebibliography}
\bibliographystyle{iclr2027_conference}

\appendix
\section{Discussion: What our theorems say about reasoning models}
\label{sec:discussion}
In this section, we connect our theorem to the capabilities of several recent kinds of reasoning models and discuss what it implies about their performance on different tasks, to provide insights on where and why they can succeed, and when they are expected to fail.

\paragraph{Chain-free models.}
Jev \citep{typesafe2026jev} is a \emph{System One} model \citep{kahneman2011thinking}, a one-pass transformer that returns a probability distribution over all possible answers. It can be viewed as a reasoning model with an empty chain, a one-pass pipeline, so Lemma~\ref{lem:pipeline} applies to quantify its ceiling. Its intended tasks, classification, routing, extraction, and guardrails, are decisions that one pass over the state is expected to make, and in our terms whose ceiling is one. On such tasks, a reasoning model pays $T$ passes for a chain whose erasure costs nothing. On an average-case serial task, however, it fails on an $\eps$ fraction of inputs that no shallow circuit could solve, and on a maximally serial task it is no better than guessing, no matter how it is trained.

\paragraph{Latent reasoning.}
Methods that internalize the chain by training with chains or distilling from a model that writes them \citep{deng2023implicit, deng2024implicit} share the same bound as the chain-free models, because such a model answers in one pass at test time. Another line of research feeds a hidden vector back in place of a token \citep{hao2024coconut} or iterates a block of layers of constant depth at test time \citep{geiping2025recurrent, saunshi2025looped}. A constant number of latent steps is a pipeline and is bounded by the ceiling in the same way (Remark~\ref{rem:latent}).

\paragraph{Filler and pause tokens.}
Compute without content does not help either. Filler tokens \citep{pfau2024} and pause tokens \citep{goyal2024pause} add positions to the context whose content is fixed in advance rather than produced by the model, so the extra computation happens inside one forward pass over a longer context. They add width and not depth, and they recover the benefit of a chain only on parallelizable tasks, whose ceiling is one.

\paragraph{Non-autoregressive and diffusion decoding.}
Non-autoregressive decoding \citep{gu2018nonautoregressive} and masked diffusion language models \citep{nie2025llada, inception2025mercury} unmask a full sequence over a number of denoising steps. That number of steps is fixed in advance and does not grow with the input, so on an average-case serial task they are capped at $1 - \eps$ and on a maximally serial task they are no better than the guessing baseline. If the number of denoising steps grows with the input, the steps take the role of the chain, and Theorem~\ref{thm:depth} applies to the sequence of drafts as it applies to a chain of tokens (Remark~\ref{rem:latent}).

\paragraph{Compressing the chain.}
Theorem~\ref{thm:depth} says what a model loses when its chain could have been written shallowly, since when $\eta_n$ is close to zero, the model cannot do better than a pipeline. As a result, no training method can compress the reasoning for a serial task into a few tokens, nor into a few latent steps (Remark~\ref{rem:latent}). The theorem also tells us which part of a chain can be shortened. A model whose whole chain is shallowly written on all but a $\lambda$ fraction of its inputs has success at most $\shal_n + \lambda$ (Lemma~\ref{lem:boilerplate} in Appendix~\ref{app:proofs-depth}), so every unit of accuracy above the ceiling is achieved from the input on which the chain is not shallowly written. As a result, compression approaches that prompt the model to produce the fewest possible steps \citep{xu2025cod} or train it to halt after its first correct answer \citep{chen2024overthinking} can be applied repeatedly until only the steps that rely on earlier steps are left. After that, the accuracy on each input whose chain is further compressed falls to $\shal_n$, the ceiling of pipeline on this task.

\paragraph{Process reward models.}
Theorem~\ref{thm:locality} tells us which chains can be graded by a shallow process reward. A model that scores a reasoning trace in one pass, such as a process reward model that grades each step \citep{lightman2023verify}, is a decoder in the sense of Definition~\ref{def:decodability}. Our theorem tells us which chains it can check. It can only check the chain if each step it grades is a shallow consequence of the input and of a bounded number of earlier steps. For example, a running product follows from the previous one and the next input element. In this case, checking the chain means checking each step locally and combining the checks with one AND, which is a shallow computation. When a chain leaves a gap between two steps that only serial work bridges, as the chains of large language models can, no shallow process reward can grade the step that hides that work without redoing it. In the extreme case where a chain provides only its final conclusion, any shallow verifier for that conclusion would also be a shallow solver, so its ability to succeed hinges on whether the problem is solvable by a shallow solver.

\section{Details for Section~\ref{sec:prelim}: one forward pass and shallow pipelines}
\label{app:onepass}

\paragraph{Conventions.}
A Boolean circuit is a directed acyclic graph of gates. Its \emph{size} is the number of gates and its \emph{depth} is the length of the longest path from an input to the output. A gate with \emph{unbounded fan-in} reads all $n$ inputs at once. A majority gate outputs $1$ whenever strictly more than half of its inputs are $1$, which enables counting within a single circuit layer. A \emph{randomized} circuit takes an auxiliary uniformly random string. By Adleman's argument \citep{adleman1978}, a randomized circuit family can be turned into a non-uniform deterministic family of the same class if its error on every input is below $2^{-|x| \cdot \omega(1)}$, and in particular below the inverse of the number of inputs. This is because a single random string that is good for every input exists and can be hardwired. We use $\omega(1)$ to denote a quantity that grows without bound as $n \to \infty$, and a function $\negl(n)$ is negligible if it is $n^{-\omega(1)}$, that is, if it decays faster than every inverse polynomial. We write $\TCz/\poly$ and $\ACz/\poly$ when we want to stress that a class is non-uniform. Table~\ref{tab:classes} summarizes the three circuit classes that appear in the paper, and Figure~\ref{fig:circuits} shows the three shapes of computation behind the notion of depth.  In the definition of the ceiling in Section~\ref{sec:prelim}, the depth and the size have to be fixed before the supremum is taken, because at any single input length a large enough circuit is a lookup table that answers every input, so a supremum over all polynomial-size families would equal $1$ at every $n$. We write $\shal_n$ without the two constants when they are clear from the context, and we read such a statement as follows. An upper bound $A \le \shal_n$ on the success $A$ of a specific shallow object, such as a pipeline or a decoder, holds with $d$ and $s$ equal to the depth and the size exponent of the circuit of that object. An upper bound $\shal_n \le B_n$ on the ceiling itself holds for every fixed $d$ and $s$ and for all sufficiently large $n$. The two readings chain, since $A \le \shal_n(d,s) \le B_n$ for all sufficiently large $n$.

\begin{table}[htbp]
\centering
\footnotesize
\begin{tabular}{@{}p{0.8cm}>{\raggedright\arraybackslash}p{2.6cm}p{1.3cm}>{\raggedright\arraybackslash}p{4.2cm}>{\raggedright\arraybackslash}p{3.2cm}@{}}
\toprule
Class & Gates & Depth & Contains & Does not contain \\
\midrule
$\ACz$ & AND, OR, NOT with unbounded fan-in & $O(1)$ & OR of $n$ bits, addition of two $n$-bit numbers & parity and majority (proved) \\
\addlinespace
$\TCz$ & the same gates and majority & $O(1)$ & counting, sums of $n$ numbers, multiplication, division, sorting, parity, word problems of solvable groups & the word problem of $A_5$ if and only if $\TCz \neq \NCone$ (open) \\
\addlinespace
$\NCone$ & AND, OR, NOT with fan-in two & $O(\log n)$ & the word problem of every finite group, evaluation of Boolean formulas & \\
\bottomrule
\end{tabular}
\caption{The three circuit classes used in this paper. All classes have polynomial size. One forward pass of a transformer lies in $\TCz$ at log precision and in $\ACz$ at constant precision (Fact~\ref{fact:onepass}).}
\label{tab:classes}
\end{table}

\begin{figure}[htbp]
\centering
\includegraphics[width=0.9\linewidth]{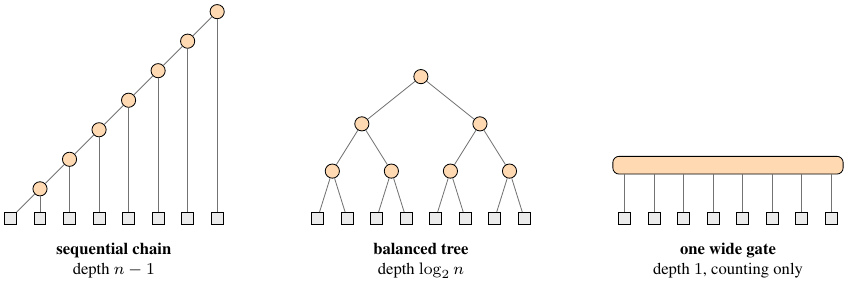}
\caption{Three circuits that combine $n = 8$ inputs. The sequential chain has depth $n - 1$ and is what a chain of thought implements, one step per token. The balanced tree has depth $\log_2 n$ and is how a circuit in $\NCone$ multiplies $n$ elements of any finite group. The single wide gate has depth $1$ and exists only when the answer depends on counts, which is the case for abelian groups and is the reason their word problems lie in $\TCz$.}
\label{fig:circuits}
\end{figure}

\begin{remark}[Sampling conventions]
\label{rem:sampling}
We model the sampling procedure for $\pi$ as a finite-precision program, which takes as input the logits together with a block of random bits and outputs a token. The sampler described in Facts~\ref{fact:tc0} and~\ref{fact:ac0} samples exactly from the distribution induced by $\pi$, since they draw by comparing a random value to the cumulative distribution. If instead one wishes to assume ideal sampling from the exact real-valued softmax, the corresponding finite-precision sampler is within total variation distance $\negl(n)$, so all claims in this paper continue to hold with an extra additive $\negl(n)$ term, which is absorbed by the existing slack.
\end{remark}

Fact~\ref{fact:onepass} is the log-precision case of the following two statements.

\begin{fact}[Log-precision setting; \citealp{merrillsabharwal2023parallelism}]
\label{fact:tc0}
Every function computed by one forward pass of a constant-depth, polynomial-width transformer with $O(\log n)$-bit precision on inputs of length $\poly(n)$ lies in $\TCz/\poly$. Moreover, sampling one token from the output distribution, given a uniformly random string, is computable by a randomized $\TCz$ circuit (compute the logits, the cumulative distribution, and a comparison).
\end{fact}

\begin{fact}[Constant-precision setting; \citealp{li2024cot}, Thm.~3.1]
\label{fact:ac0}
In the finite-precision model of \citet{li2024cot} (constant-bit mantissa and exponent, rounding after each operation), every function computed by one forward pass of a constant-depth transformer lies in $\ACz/\poly$, and sampling one output token is computable by a randomized $\ACz$ circuit.
\end{fact}

Precision refers to the number of bits with which the transformer stores each number in its computation, such as an activation or an attention score. A transformer with $O(\log n)$ bits can represent an average over all $n$ positions accurately enough to count, which is what a majority gate does, whereas a transformer with constant precision cannot. This is why the two settings land in $\TCz$ and in $\ACz$. The log-precision setting is the standard model in this literature and the one that the main text uses, and we use the word \emph{shallow} for the class that matches the setting.

\begin{proof}[Proof of Lemma~\ref{lem:pipeline}]
Fix $n$ and a pipeline with $q$ stages. We show by induction over the stages that the output of every stage is computed by a randomized circuit of polynomial size and constant depth in the class of the setting. The output of every stage is a function of $x$ and of the random bits used so far. For the base case, any stage that uses a shallow randomized map satisfies the claim, since composing it with the circuits from the preceding stages preserves the result. A stage that runs a forward pass consists of three steps. First, it assembles a context of length $\poly(n)$ from $x$ and from the earlier outputs. The layout of the context is fixed in advance, so this step only copies symbols and has constant depth. Second, it runs one forward pass of $\pi$ on this context. The forward pass on a context of length $\poly(n)$ is in $\TCz$ by Fact~\ref{fact:tc0} and in $\ACz$ by Fact~\ref{fact:ac0}, and the softmax with its normalization is part of the cited containments. Third, it samples one token at every selected position. Given the probabilities $(p_1, \dots, p_{|\Sigma|})$ at the working precision, the circuit computes the prefix sums $s_j = \sum_{i \leq j} p_i$, interprets a fresh block of random bits as a number $r' \in [0,1)$ at the working precision, and outputs the first $j$ with $r' < s_j$. The prefix sums are an iterated addition of polynomially many $O(\log n)$-bit numbers, which is in $\TCz$, and in the constant-precision setting the alphabet is constant and the additions have constant size, so they are in $\ACz$. The comparisons and the selection are a constant number of operations over a constant alphabet. The selected positions are handled in parallel by disjoint copies of this sampler with independent random bits, which does not increase the depth. Each stage therefore has polynomial size and constant depth, and composing at most $q = O(1)$ such stages preserves both. The depth of the composed circuit is the sum of the depths of its stages and its size is the sum of their sizes, so the composed circuit lies in $\mathcal{C}_n(d,s)$ for constants $d$ and $s$ that depend only on the depth and the size exponent of the circuits of Facts~\ref{fact:tc0} and~\ref{fact:ac0} for $\pi$, on those of the shallow maps of the pipeline, and on $q$. By Remark~\ref{rem:sampling} the output of the resulting circuit has exactly the distribution of the answer of the pipeline. Its success is the probability, over a random input and the random bits, that its output equals $y(x)$, which is at most $\shal_n(d,s)$ by the definition of the ceiling in Section~\ref{sec:prelim}.
\end{proof}

\begin{remark}[Multi-token answers]
\label{rem:multitoken}
When the output contains $m = O(1)$ tokens generated autoregressively after the chain, we simply unroll the $m$ successive forward passes. The resulting composite has depth $O(mL) = O(1)$, remains within the stated class, and by Definition~\ref{def:pipeline}, it just appends $m$ stages to the pipeline. When $m = \omega(1)$, producing the answer introduces its own sequential cost that must be included in $T$, and the theorems then hold with $T + m$ substituted for $T$.
\end{remark}

\begin{remark}[Latent chains]
\label{rem:latent}
Neither Fact~\ref{fact:tc0}, Definition~\ref{def:pipeline}, nor Lemma~\ref{lem:pipeline} relies on the assumption that the item passed from one forward pass to the next is a token. Assume instead that a pass outputs a vector containing $\poly(n)$ values, each represented with $O(\log n)$ bits, and that the subsequent pass takes this vector as an input embedding. The function mapping the context to this vector is just the forward pass with the sampling step removed, and therefore it is in $\TCz$ by Fact~\ref{fact:tc0}. By the same inductive argument, a pipeline whose stages output these vectors forms a shallow circuit, and an erasing intervention that swaps the entire vector sequence for a shallow function of $x$ reduces the system to a two-stage pipeline. Repeating a layer block $r$ times is a special case where the vector is the complete hidden state and the context length does not increase. Hence Theorem~\ref{thm:necessity} and Theorem~\ref{thm:depth} also apply when the chain $c$ is replaced by the sequence of latent states, with total variation distances computed over the discretized vectors. Each latent step can encode $\poly(n)$ bits, whereas a token carries only $O(\log|\Sigma|)$ bits; this changes how many steps a task may require (Conjecture~\ref{conj:length}) but not whether the required number of steps must scale with $n$.
\end{remark}

\subsection{The constant-precision setting and parity}
\label{app:parity}
\label{app:constprec}

Under constant precision, Fact~\ref{fact:ac0} implies that a forward pass lies in $\ACz$, meaning that “shallow” corresponds to $\ACz$. In this setting, we take parity as an example of inherently serial task according to Definition~\ref{def:serial}.

\begin{definition}[Parity]
\label{def:parity}
In the parity task $\PAR_n$, the input is $x \sim \Unif(\{0,1\}^n)$, the target is $y(x) = x_1 \oplus \cdots \oplus x_n$, where $\oplus$ denotes addition modulo $2$, and the guessing baseline is $b = 1/2$.
\end{definition}

\begin{lemma}[Average-case hardness of parity; unconditional]
\label{lem:parity}
There is a constant $\kappa > 0$ such that for every randomized Boolean circuit family of size $s(n) = \poly(n)$ and constant depth $d$,
$\Pr_{x \sim \Unif(\{0,1\}^n),\, r}[C_n(x;r) = \PAR_n(x)] \leq \tfrac{1}{2} + 2^{-\kappa n / (\log s(n))^{d-1}} = \tfrac12 + 2^{-n/\polylog(n)}$.
\end{lemma}

\begin{proof}
For deterministic circuits this is the correlation bound of \citet{hastad2014}; see also \citet{imp2012} and \citet{hastad1987}. For randomized circuits, the success is an average over the random string of the successes of deterministic circuits, so it is at most their maximum.
\end{proof}

Parity is the right task for this setting because it is hard for $\ACz$ and easy for $\TCz$, where a majority gate counts the ones. The word problem of $A_5$ is the right task for the log-precision setting for the mirror-image reason. In the terms of Definition~\ref{def:serial}, Lemma~\ref{lem:parity} says that parity on uniform inputs is maximally serial for $\ACz$, unconditionally, where the passage from a bound for every family to a bound on $\shal_n(d,s)$ is the one given after Lemma~\ref{lem:avgcase}. At log precision parity is not even inherently serial, since a $\TCz$ circuit computes it on every input. Every result in Sections~\ref{sec:prelim} to~\ref{sec:locality} is proved from three ingredients only, namely Fact~\ref{fact:onepass}, Lemma~\ref{lem:pipeline}, and the ceiling bound of Lemma~\ref{lem:avgcase}. If we replace them by Fact~\ref{fact:ac0}, the $\ACz$ case of Lemma~\ref{lem:pipeline}, whose proof above covers both settings, and Lemma~\ref{lem:parity}, we obtain the constant-precision version of every statement. In that version shallow means $\ACz$, the task is $\PAR_n$, the guessing baseline is $\tfrac12$, the ceiling is at most $\tfrac12 + 2^{-n/\polylog(n)}$, and the conjecture of Appendix~\ref{app:avgcase} is not needed. The interventions of Section~\ref{sec:necessity} lie in both classes. The two constructions in Appendix~\ref{app:proofs-locality} that use counting, Lemma~\ref{lem:bag} and Proposition~\ref{prop:dissoc}, are stated for the log-precision setting only.

\section{Details for Section~\ref{sec:exp-toy}: average-case hardness of the word problem of $A_5$}
\label{app:avgcase}

\begin{fact}[\citealp{barrington1989, barringtontherien1988}]
\label{fact:barrington}
For every fixed finite nonsolvable group $G$, the iterated word problem $\WP_G$ (given $(g_1, \dots, g_n) \in G^n$, output the product $g_1 g_2 \cdots g_n$) is $\NCone$-complete under $\ACz$ (indeed projection) reductions. In particular this holds for $G = A_5$, the alternating group on five elements, which is simple and nonsolvable, with $|A_5| = 60$.
\end{fact}

\paragraph{The conjecture.}
All log-precision claims about $\WP_n$ rely on the standard conjecture $\NCone \not\subseteq \TCz$, which is the sole unproved assumption used in this paper. By Fact~\ref{fact:barrington}, this conjecture is equivalent to $\WP_{A_5} \notin \TCz/\poly$. Indeed, if there were a constant-depth, polynomial-size family of threshold circuits that computed the product correctly on every input length, then the projection reductions in Fact~\ref{fact:barrington} would imply $\NCone \subseteq \TCz/\poly$; conversely, $\WP_{A_5}$ is contained in $\NCone$. Interpreted for a single circuit family, the conjecture asserts that any such family must fail on some length-$n$ input for infinitely many $n$. The results in this appendix instead assume the almost-everywhere version, which requires the same conclusion for all sufficiently large $n$.

\begin{assumption}[Almost-everywhere hardness of the word problem of $A_5$]
\label{asm:worstcase}
For every polynomial-size threshold circuit family $\{C_n\}$ of constant depth and all sufficiently large $n$, there is an input $x \in A_5^n$ with $C_n(x) \neq y(x)$.
\end{assumption}

Assumption~\ref{asm:worstcase} is slightly more restrict than the conjecture we use in order to keep the quantifiers tidy. If it were false, then some $\NCone$-complete problem would admit constant-depth circuits for infinitely many input lengths. This is the version implicitly meant in the main text whenever a bound is stated for $\WP_n$. With the conjecture in its usual formulation, each bound stated below would hold for infinitely many $n$ rather than for all sufficiently large $n$, and that is enough for all of the paper’s conceptual takeaways (Appendix~\ref{app:quantifiers}).

\begin{lemma}[Average-case hardness of the word problem of $A_5$; conditional]
\label{lem:avgcase}
Under Assumption~\ref{asm:worstcase}, for every polynomial-size randomized $\TCz$ family $\{C_n\}$ and every constant $k$, for all sufficiently large $n$:
\[
\Pr_{x \sim \Unif(A_5^n),\ r}\big[C_n(x; r) = y(x)\big] \;\leq\; \frac{1}{60} + n^{-k}.
\]
\end{lemma}

Lemma~\ref{lem:avgcase} upper-bounds the best achievable success probability for $\WP_n$ at any fixed depth and size. Suppose $\shal_n(d,s) > \tfrac{1}{60} + n^{-k}$ occurred for infinitely many $n$. Then, choosing for each such $n$ a circuit in $\mathcal{C}_n(d,s)$ that succeeds above this threshold, and selecting any circuit of the same depth and size for the remaining values of $n$, would yield a polynomial-size randomized $\TCz$ family contradicting the lemma. Therefore, for all $d,s$ and any constant $k$, we have $\shal_n(d,s) \le \tfrac{1}{60} + n^{-k}$ for all sufficiently large $n$. Consequently, $\WP_n$ under the uniform distribution is maximally serial as in Definition~\ref{def:serial}(iii), and this is the bound invoked in Section~\ref{sec:exp-toy} for $\WP_n$.

The ceiling in Lemma~\ref{lem:avgcase} is induced from Theorem~3.9 of \citet{milesviola2013}. That theorem says that if $\TCz \neq \NCone$, then for every $k$ and for infinitely many $t$, the family of $\TCz$ circuits of size at most $t^k$ with $k \log t$ output bits is $t^{-k}$-fooled by $(A_5)^t$. A class is said to be $\eps$-fooled if no circuit in it has statistical distance $\eps$ between its output distribution under the uniform distribution on $G^t$ and its output distribution under the uniform distribution on $t$-tuples whose product is $\alpha$, for any choice of $\alpha$. Converting their formulation into the prediction version is just a single averaging argument, which we spell out. Assume a $\TCz$ family $C$ predicts the product on uniform inputs with success probability $1/m + \delta$. For each $\alpha \in G$, let $D_\alpha$ be the uniform distribution on $n$-tuples with product $\alpha$, set $p_\alpha := \Pr_{x \sim D_\alpha}[C(x) = \alpha]$, and define $q_\alpha := \Pr_{x \sim \Unif(G^n)}[C(x) = \alpha]$, the overall frequency with which $C$ outputs $\alpha$. Because $\Unif(G^n)$ is the uniform mixture $\tfrac1m \sum_\alpha D_\alpha$, we have $\tfrac1m \sum_\alpha p_\alpha = \tfrac1m + \delta$, so $\sum_\alpha p_\alpha = 1 + m\delta$, whereas $\sum_\alpha q_\alpha \leq 1$. Therefore, for some $\alpha$ we must have $p_\alpha - q_\alpha \geq \delta$, and for that $\alpha$ the test “output equals $\alpha$” distinguishes $D_\alpha$ from $\Unif(G^n)$ with advantage $\delta$, contradicting $\delta$-fooling. Their theorem is stated for infinitely many $t$, while our version ranges over all sufficiently large $n$; Appendix~\ref{app:quantifiers} explains how to align these quantifier conventions. The randomization in Step~1 of the proof below is the random self-reduction from their Lemma~3.2, applied here in a two-sided way. The remainder proceeds differently: we use a two-sided self-reduction with profile matching, completed by a covering argument. We retain this approach because it produces the prediction statement directly, and because Remark~\ref{rem:simplicity} extracts that line of reasoning.

\medskip\noindent\emph{Proof of Lemma~\ref{lem:avgcase}.} Throughout, $G = A_5$, $m = |G| = 60$, and $y(x) = g_1 \cdots g_n$ for $x = (g_1, \dots, g_n) \in G^n$. Suppose, for contradiction, that there is a polynomial-size randomized $\TCz$ family $\{C_n\}$, a constant $k$, and an infinite set $N \subseteq \mathbb{N}$ such that for all $n \in N$,
\[
s_n \;:=\; \Pr_{\wt x \sim \Unif(G^n),\ \rho}\big[C_n(\wt x; \rho) = y(\wt x)\big] \;\geq\; \frac{1}{m} + \delta, \qquad \delta := n^{-k}.
\]
We construct, for each $n \in N$, a polynomial-size deterministic $\TCz$ circuit that computes $\WP_n$ exactly on all inputs. Since $N$ is infinite, this contradicts Assumption~\ref{asm:worstcase}.

\paragraph{Step 1: two-sided random self-reduction.}
Fix an arbitrary target instance $x = (g_1, \dots, g_n)$. Draw $h_0, h_1, \dots, h_n \sim \Unif(G)$ i.i.d.\ and define $\wt x = (\wt x_1, \dots, \wt x_n)$ by $\wt x_i := h_{i-1}^{-1} g_i h_i$. Telescoping,
\[
y(\wt x) = h_0^{-1} g_1 h_1 \cdot h_1^{-1} g_2 h_2 \cdots h_{n-1}^{-1} g_n h_n = h_0^{-1}\, y(x)\, h_n,
\qquad\text{so}\qquad
y(x) = h_0\, y(\wt x)\, h_n^{-1}.
\]
Define the estimate $\hat y := h_0\, C_n(\wt x; \rho)\, h_n^{-1}$ with fresh coins $\rho$.

We use two distributional facts. The first is that $\wt x$ is uniform and independent of $h_0$. For each fixed value of $h_0$ and the fixed $x$, the map $(h_1, \dots, h_n) \mapsto \wt x$ is a bijection of $G^n$, since it can be inverted coordinatewise by $h_i = g_i^{-1} h_{i-1} \wt x_i$, so the conditional law of $\wt x$ given $h_0$ is uniform for every $h_0$. The second is that the offset is a conjugated copy of an error that does not depend on the instance. Writing $\eps := C_n(\wt x; \rho)\, y(\wt x)^{-1}$,
\[
\hat y\, y(x)^{-1}
= h_0\, C_n(\wt x;\rho)\, h_n^{-1} \cdot \big(h_0\, y(\wt x)\, h_n^{-1}\big)^{-1}
= h_0\, C_n(\wt x;\rho)\, y(\wt x)^{-1}\, h_0^{-1}
= h_0\, \eps\, h_0^{-1}.
\]
By the first fact, $\eps$ is a function of $(\wt x, \rho)$ only and hence independent of $h_0$, and $\eps \sim D$ where $D(g) := \Pr_{\wt x \sim \Unif, \rho}[\,C_n(\wt x;\rho)\, y(\wt x)^{-1} = g\,]$ is a fixed distribution on $G$ depending only on $(C_n, n)$, with $D(e) = s_n$.

Consequently the offset $\hat y\, y(x)^{-1}$ is distributed according to the \emph{conjugation average}
\[
\bar D(g) \;:=\; \frac{1}{m} \sum_{h \in G} D\big(h^{-1} g\, h\big),
\]
which is a class function, that is, $\bar D(g) = \bar D(w g w^{-1})$ for all $w$, satisfies $\bar D(e) = D(e) = s_n$, and does not depend on the target instance $x$. The last property is what the following steps use.

\paragraph{Step 2: repetition and the empirical profile.}
Let $c = c(G) \leq m$ be the covering constant of Lemma~\ref{lem:covering}, set $\gamma := \delta / (16c)$, and repeat Step~1 independently $k^* := \big\lceil 8 \gamma^{-2} \big( (n+2) \ln m + \ln 8 \big) \big\rceil = O\big(c^2 n / \delta^2\big) = \poly(n)$ times, obtaining estimates $\hat y^{(1)}, \dots, \hat y^{(k^*)}$, i.i.d.\ with $\hat y^{(j)} = e^{(j)} \cdot y(x)$ and $e^{(j)} \sim \bar D$. Form the empirical profile $\hat q(a) := \frac{1}{k^*} \sum_j \ind[\hat y^{(j)} = a]$ for each $a \in G$. Since $\E\, \hat q(a) = \bar D(a\, y(x)^{-1})$, Hoeffding and a union bound over the $m$ atoms give
\[
\Pr\Big[\ \exists a:\ \big|\hat q(a) - \bar D(a\, y(x)^{-1})\big| > \gamma\ \Big] \;\leq\; 2 m\, e^{-2 \gamma^2 k^*} \;\leq\; \tfrac{1}{4}\, m^{-(n+1)}.
\]
Call the complementary event \textsc{Good}.

\paragraph{Step 3: profile matching with hardwired advice.}
The non-uniform advice of the circuit includes a table $\wt{D}$ with $|\wt{D}(g) - \bar D(g)| \leq \gamma$ for all $g$. Such a table exists because $\bar D$ does not depend on the instance and has $m$ entries, and storing each entry to $O(\log(1/\gamma)) = O(\log n)$ bits suffices. The decoder outputs
\[
v^\star \;\in\; \arg\min_{v \in G}\ \Delta(v), \qquad \Delta(v) := \max_{a \in G} \big|\hat q(a) - \wt{D}(a\, v^{-1})\big|,
\]
breaking ties arbitrarily. On \textsc{Good}, $\Delta(y(x)) \leq 2\gamma$, since the sampling error is at most $\gamma$ and the table error is at most $\gamma$.

\paragraph{Step 4: simplicity forces uniqueness of the match.}
Suppose that on \textsc{Good} the minimizer were some $v \neq y(x)$. Then $\Delta(v) \leq \Delta(y(x)) \leq 2\gamma$, hence $|\hat q(a) - \bar D(a\, v^{-1})| \leq 3\gamma$ for all $a$ after removing the table error, and the triangle inequality against $|\hat q(a) - \bar D(a\, y(x)^{-1})| \leq \gamma$ gives, for all $a \in G$,
\[
\big|\bar D(a\, y(x)^{-1}) - \bar D(a\, v^{-1})\big| \leq 4\gamma.
\]
Substituting $a = b\, y(x)$ and setting $u := y(x)\, v^{-1} \neq e$: for all $b \in G$,
\begin{equation}
\label{eq:translation}
\big|\bar D(b) - \bar D(b\, u)\big| \leq 4\gamma .
\end{equation}
We now extend \eqref{eq:translation} from $u$ to all of $G$ in three steps. The first step extends it to $u^{-1}$. Substituting $b \mapsto b u^{-1}$ in \eqref{eq:translation} gives $|\bar D(b u^{-1}) - \bar D(b)| \leq 4\gamma$ for all $b$. The second step extends it to every conjugate $w = h u h^{-1}$, and likewise to $h u^{-1} h^{-1}$. Since $\bar D$ is a class function,
\[
\bar D(b\, h u h^{-1}) = \bar D\big(h^{-1} b\, h\, u\big), \qquad \bar D(b) = \bar D\big(h^{-1} b\, h\big),
\]
so $|\bar D(b w) - \bar D(b)| = |\bar D(b' u) - \bar D(b')| \leq 4\gamma$ with $b' = h^{-1} b h$. Thus \eqref{eq:translation} holds with any element of the set $K$ of Lemma~\ref{lem:covering} in place of $u$. The third step extends it to an arbitrary $g \in G$. Writing $g = w_1 \cdots w_\ell$ with $w_j \in K$ and $\ell \leq c$ by Lemma~\ref{lem:covering} and telescoping,
\[
\big|\bar D(b\, g) - \bar D(b)\big| \;\leq\; \sum_{j=1}^{\ell} \big|\bar D(b\, w_1 \cdots w_j) - \bar D(b\, w_1 \cdots w_{j-1})\big| \;\leq\; 4 c \gamma \qquad \text{for all } b, g \in G.
\]
Taking $b = e$ shows $\bar D$ is within $4c\gamma$ of the constant $\bar D(e)$ pointwise; summing over $G$ and using $\sum_g \bar D(g) = 1$,
\[
\Big| \bar D(e) - \tfrac{1}{m} \Big| \;\leq\; 4 c \gamma \;=\; \frac{\delta}{4}.
\]
But $\bar D(e) = s_n \geq \tfrac1m + \delta$, a contradiction. Hence on \textsc{Good} the unique minimizer is $v^\star = y(x)$, and the reduction errs with probability at most $\tfrac14 m^{-(n+1)}$ \emph{on every input $x$}.

\paragraph{Step 5: complexity of the wrapper and derandomization.}
The wrapper implementing the $k^*$ oracle invocations consists of five components. The coordinate-wise products $h_{i-1}^{-1} g_i h_i$ can be computed by table lookups on triples of symbols from a $60$-symbol alphabet, and therefore fall in $\mathsf{NC}^0$; the same holds for the final multiplications $h_0 (\cdot) h_n^{-1}$. The $k^*$ instances of $C_n$ are executed in parallel, each having polynomial size and identical constant depth. For the $m$ atoms, the quantities $k^* \hat q(a)$ are obtained via iterated addition of $k^*$ bits using majority gates, placing this part in $\TCz$. Comparing $O(\log n)$-bit fixed-point values to the hardwired table, as well as forming the $m^2$ absolute differences $|\hat q(a) - \bar D(a v^{-1})|$, can be done in $\ACz$. Taking the argmin over the $m=60$ possibilities is merely a constant-size selection step. Altogether, the resulting circuit has constant depth and polynomial size, hence is in $\TCz$, and it uses randomness from $(h^{(j)}_i)_{i,j}$ together with the oracle’s random bits. For any fixed input its error is at most $\tfrac14 m^{-(n+1)} < m^{-n}$. As there are exactly $m^n$ inputs, a union bound implies that a uniformly random choice of the random string works simultaneously for all inputs with probability at least $3/4$. Therefore a good string exists, and fixing it in the circuit (as in Adleman’s method) gives a deterministic $\TCz/\poly$ family computing $\WP_n$ exactly for every $n \in N$. By Fact~\ref{fact:barrington}, this would put an $\NCone$-complete problem in $\TCz/\poly$ for infinitely many input lengths, contradicting Assumption~\ref{asm:worstcase}. \hfill $\qed$

\begin{remark}[Where simplicity enters, and where $S_5$ fails]
\label{rem:simplicity}
The only place where the proof uses the structure of $G$ is Step~4. The passage from near-invariance under one $u \neq e$ to near-invariance under all of $G$ requires that the conjugates of $u^{\pm 1}$ generate $G$, that is, that the normal closure of every nontrivial element is the whole group, which is exactly simplicity. For $G = S_5$ and $u \in A_5$ the normal closure is $A_5 \neq S_5$, so the telescoping argument only forces $\bar D$ to be constant on the cosets of $A_5$, and no contradiction arises at $s = 2/|S_5| = 1/60$. This is the correct behavior, because Proposition~\ref{prop:shortcut} exhibits a $\TCz$ circuit that reaches exactly this value by computing the sign of the product.
\end{remark}

\subsection{The covering lemma}
\label{app:covering}

\begin{lemma}[Covering in finite simple groups]
\label{lem:covering}
Let $G$ be a finite nonabelian simple group and $u \in G \setminus \{e\}$. Let $K = \{g u g^{-1} : g \in G\} \cup \{g u^{-1} g^{-1} : g \in G\}$. Then there exists $c = c(G) \leq |G|$ such that every element of $G$ is a product of at most $c$ elements of $K$.
\end{lemma}

\begin{proof}
The set $K$ is symmetric, since the inverse of a conjugate of $u$ is the same conjugate of $u^{-1}$, and it is closed under conjugation by construction. Let $H := \langle K \rangle$. Since conjugation permutes $K$, it maps $H$ to itself, so $H \trianglelefteq G$, and $H \neq \{e\}$ since $u \in K$. By simplicity, $H = G$. Now let $S_j := \{w_1 \cdots w_i : 0 \leq i \leq j,\ w_1, \dots, w_i \in K\}$ denote the set of products of at most $j$ factors from $K$, with the empty product giving $e \in S_0$. The chain $S_0 \subseteq S_1 \subseteq \cdots$ is nondecreasing. If $S_{j+1} = S_j$ for some $j$, then $S_j K \subseteq S_j$, and by induction $S_j$ contains every finite product of elements of $K$. Since $K$ is symmetric and $G$ is finite, the set of such products is the subgroup $\langle K \rangle = G$, so $S_j = G$. Each strict inclusion in the chain adds at least one element of the finite group, so stabilization, and hence $S_j = G$, occurs at some $j \leq |G|$. Taking $c := |G|$ proves the claim. The known covering numbers of $A_5$ are far smaller, but the proof of Lemma~\ref{lem:avgcase} only needs some constant, so we keep the self-contained argument.
\end{proof}

\subsection{The word problem at each degree}
\label{app:degrees}
The example of Section~\ref{sec:exp-toy} places the word problem at each of the three degrees of Definition~\ref{def:serial}, depending on the group and on the input distribution. We collect the four statements and their proofs here. Throughout this subsection shallow means $\TCz$, the class of the log-precision setting.

For a commutative group, such as $\mathbb{Z}_{60}$, the cyclic group of the integers modulo $60$, a shallow circuit computes the product on every input, so the ceiling is $1$ and the task is not serial in any sense. For $A_5$ (the $60$ even permutations of five objects) under a distribution that puts almost all of its mass on inputs with few non-identity elements, no shallow circuit computes the product on every input, so the task is inherently serial, but a shallow circuit solves every input with few non-identity elements, which carry almost all of the mass, so the ceiling tends to $1$ (Proposition~\ref{prop:skewed}). For $S_5$, the symmetric group of all $120$ permutations of five objects, a shallow circuit is able to compute the sign of the product by counting the number of odd factors with majority gates, and guesses the rest, so the ceiling is $2b = 1/60$ and the task is average-case serial. For $A_5$ on uniform inputs, no shallow circuit predicts the product better than guessing \citep{milesviola2013}, so the ceiling is $b = 1/60$ and the task is maximally serial. 

\paragraph{Solvable groups: no degree.}
If $G$ is solvable, and in particular if it is commutative, then $\WP_G$ lies in $\TCz$ \citep{barringtontherien1988}, so a shallow circuit computes the product on every input, the ceiling is one under every input distribution, and none of the three conditions holds. This covers the two control groups $\mathbb{Z}_{60}$ and $A_4 \times \mathbb{Z}_5$ of Section~\ref{sec:experiments}.

\paragraph{$S_5$ on uniform inputs: average-case serial, not maximally serial.}
The sign of a product is the parity of the signs of its factors, which a $\TCz$ circuit computes by counting the odd permutations, and a circuit that outputs a fixed element of the correct coset of $A_5$ is right with probability $1/|A_5| = 1/60 = 2/|S_5|$ (Proposition~\ref{prop:shortcut}, proved in Appendix~\ref{app:proofs-necessity}). Under Assumption~\ref{asm:worstcase}, no polynomial-size randomized $\TCz$ family exceeds $1/60 + n^{-k}$ on uniform inputs from $S_5^n$, by the variant of the proof of Lemma~\ref{lem:avgcase} given in the same appendix. Hence $\shal_n \leq 1/60 + n^{-k}$, which is bounded away from one, so the task is average-case serial with $\eps = 1 - 1/60 - o(1)$; and $\shal_n \geq 1/60 = 2b$, so it is not maximally serial.

\paragraph{$A_5$ on uniform inputs: maximally serial.}
This is Lemma~\ref{lem:avgcase}.

\paragraph{$A_5$ under a skewed distribution: inherently serial, not average-case serial.}
The first degree does not depend on the input distribution, and the second does. The following proposition makes the difference concrete on the same group.

\begin{proposition}[The distribution decides the degree]
\label{prop:skewed}
Let $E_n \subseteq A_5^n$ be the set of inputs with at most three entries different from the identity, and let $D_n$ be any distribution on $A_5^n$ with $D_n(E_n) \geq 1 - 2^{-n}$. A deterministic $\TCz$ family computes the product on every input in $E_n$, so the ceiling of $\WP_{A_5}$ under $D_n$ is at least $1 - 2^{-n}$ and the task is not average-case serial, whereas under Assumption~\ref{asm:worstcase} it remains inherently serial.
\end{proposition}

\begin{proof}
For each position $i$ the circuit computes the bit $z_i := \ind[g_i \neq e]$ by a lookup and the prefix count $c_i := \sum_{j \leq i} z_j$, an iterated addition of $n$ bits, which lies in $\TCz$. For $j \in \{1, 2, 3\}$ it forms the one-hot encoding of the $j$-th non-identity entry as the bitwise OR, over all positions $i$, of the encoding of $g_i$ gated by $\ind[z_i = 1 \wedge c_i = j]$. On an input in $E_n$ at most one position passes each gate, so the OR returns exactly that entry, and an all-zero result is read as $e$. The circuit then outputs the product of the three retrieved entries, a lookup on three symbols. On every input in $E_n$ this product equals $y(x)$, because the identity entries contribute nothing to $g_1 \cdots g_n$. The circuit has polynomial size and constant depth, so its accuracy under $D_n$ is at least $D_n(E_n) \geq 1 - 2^{-n}$, and the ceiling is at least that; it is not bounded away from one, so Definition~\ref{def:serial}(ii) fails. Definition~\ref{def:serial}(i) does not involve the distribution, and Assumption~\ref{asm:worstcase} says exactly that every $\TCz$ family errs on some input of $A_5^n$ for all sufficiently large $n$, so the task is inherently serial. The number three is not essential, and the same construction handles any constant number of non-identity entries.
\end{proof}

\subsection{On almost-everywhere versus infinitely-often hardness}
\label{app:quantifiers}
Assumption~\ref{asm:worstcase} is stated almost everywhere so that Lemma~\ref{lem:avgcase} holds for all sufficiently large $n$. Under the conjecture in its standard form, which by Fact~\ref{fact:barrington} says that $\WP_{A_5} \notin \TCz/\poly$, that is, that every family fails on infinitely many $n$, the same proof shows that the bound of Lemma~\ref{lem:avgcase} holds for infinitely many $n$. The results of Sections~\ref{sec:necessity} to~\ref{sec:locality} then hold along that subsequence, which suffices for every qualitative conclusion of this paper, since the experiments operate at fixed $n$ in any case. We use the almost-everywhere form because it reads more easily.
 \section{Proofs for Section~\ref{sec:necessity}}
\label{app:proofs-necessity}

\paragraph{Formal setting.}
We follow the conventions from Appendix~\ref{app:onepass}. An intervention is a family $\Phi=\{\Phi_n\}$ of randomized maps $\Phi_n:\Sigma^n \times \Sigma^{T(n)} \to \Sigma^{T'(n)}$ with $T'(n)\le \poly(n)$, such that $\Phi_n$ can be computed from $(x,c)$ together with the intervention’s random string by a constant-depth, polynomial-size randomized circuit in the relevant class for the setting. The specific interventions we study—constants, truncation, independent random tokens, and shuffling—belong to both classes. We call an intervention erasing if $\Phi_n(x,c)=\varphi_n(x;r)$ for some randomized map $\varphi_n$ in the same class and some random string $r$. Figure~\ref{fig:interventions} depicts the protocol along with these two types of interventions.

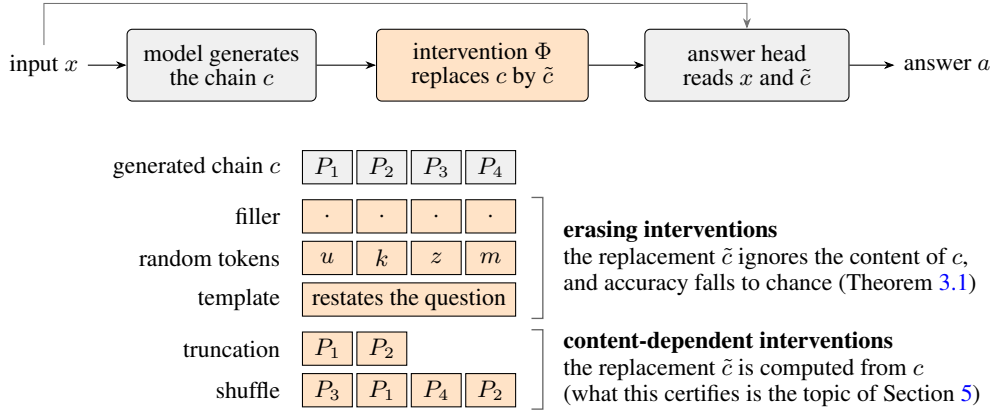
\begin{figure}[htbp]
\centering
\begin{tikzpicture}[>={Stealth[length=4pt]}, font=\small,
  box/.style={draw, rounded corners=2pt, minimum height=0.95cm, align=center, inner sep=3pt, fill=black!6},
  phi/.style={box, fill=orange!22},
  tokc/.style={draw, minimum width=0.66cm, minimum height=0.44cm, inner sep=0pt, font=\footnotesize, fill=black!6},
  toke/.style={tokc, fill=orange!22},
  lab/.style={anchor=east, font=\footnotesize},
  note/.style={anchor=west, font=\footnotesize, align=left},
  brk/.style={black!60, line width=0.5pt}]
  \node (x) at (0.1,0) {input $x$};
  \node[box, minimum width=2.5cm] (gen) at (2.45,0) {model generates\\[-1pt]the chain $c$};
  \node[phi, minimum width=2.8cm] (phi) at (5.9,0) {intervention $\Phi$\\[-1pt]replaces $c$ by $\tilde c$};
  \node[box, minimum width=2.7cm] (head) at (9.4,0) {answer head\\[-1pt]reads $x$ and $\tilde c$};
  \node (a) at (12.05,0) {answer $a$};
  \draw[->] (x) -- (gen); \draw[->] (gen) -- (phi); \draw[->] (phi) -- (head); \draw[->] (head) -- (a);
  \draw[->, black!55] (x.north) -- ++(0,0.55) -| (head.north);
  \node[lab] at (3.35,-1.35) {generated chain $c$};
  \foreach \k/\t in {0/P_1,1/P_2,2/P_3,3/P_4} { \node[tokc] at (3.85+0.72*\k,-1.35) {$\t$}; }
  \node[lab] at (3.35,-2.00) {filler};
  \foreach \k in {0,1,2,3} { \node[toke] at (3.85+0.72*\k,-2.00) {$\cdot$}; }
  \node[lab] at (3.35,-2.55) {random tokens};
  \foreach \k/\t in {0/u,1/k,2/z,3/m} { \node[toke] at (3.85+0.72*\k,-2.55) {$\t$}; }
  \node[lab] at (3.35,-3.10) {template};
  \node[toke, minimum width=2.82cm] at (4.93,-3.10) {restates the question};
  \node[lab] at (3.35,-3.75) {truncation};
  \foreach \k/\t in {0/P_1,1/P_2} { \node[toke] at (3.85+0.72*\k,-3.75) {$\t$}; }
  \node[lab] at (3.35,-4.30) {shuffle};
  \foreach \k/\t in {0/P_3,1/P_1,2/P_4,3/P_2} { \node[toke] at (3.85+0.72*\k,-4.30) {$\t$}; }
  \draw[brk] (6.55,-1.75) -- (6.7,-1.75) -- (6.7,-3.35) -- (6.55,-3.35);
  \node[note] at (6.85,-2.55) {\textbf{erasing interventions}\\the replacement $\tilde c$ ignores the content of $c$,\\and accuracy falls to chance (Theorem~\ref{thm:necessity})};
  \draw[brk] (6.55,-3.50) -- (6.7,-3.50) -- (6.7,-4.55) -- (6.55,-4.55);
  \node[note] at (6.85,-4.02) {\textbf{content-dependent interventions}\\the replacement $\tilde c$ is computed from $c$\\(what this certifies is the topic of Section~\ref{sec:locality})};
\end{tikzpicture}
\caption{Test-time interventions. The model generates its chain $c$ as usual. The intervention $\Phi$ then replaces $c$ by a modified chain $\tilde c$, and the answer head produces the answer from the input and $\tilde c$. The load-bearing gap is the accuracy that the model loses in this way. The lower part shows examples of $\tilde c$ (orange) for a chain that consists of four running products $P_1, \dots, P_4$.}
\label{fig:interventions}
\end{figure}

\begin{proof}[Proof of Theorem~\ref{thm:necessity}]
Fix $n$ and an erasing intervention with $\tilde c = \varphi_n(x; r)$. The intervened model runs $\pi$ for $T$ passes to generate a chain $c$, exactly as the reasoning model does, but an erasing intervention discards $c$, so the model then samples $a \sim \pi(\cdot \mid x, \tilde c)$ from a context that depends on $(x, \tilde c)$ alone. Neither the law of $\tilde c$ nor the conditional law of $a$ given $(x, \tilde c)$ depends on $c$, so the $T$ passes that produced $c$ contribute nothing to the distribution of the answer, and marginalizing over $c$ leaves
\[
\suc_n(\pi^\Phi) \;=\; \Pr_{x,\, r,\, r'}\big[\,A_n(x, \varphi_n(x;r); r') = y(x)\,\big],
\]
where $A_n(x, \tilde c\,; r')$ denotes one forward pass of $\pi$ on the context formed by $x$, a separator, and $\tilde c$, which has length $\poly(n)$, followed by sampling one token from the answer head with the random string $r'$. The right-hand side is the success of a pipeline with two stages in the sense of Definition~\ref{def:pipeline}. The first stage applies the shallow map $\varphi_n$, and the second stage runs the forward pass and samples the answer. The intervened model is therefore not a pipeline as a procedure, since it performs $T(n)$ serial passes, but its answer has the same distribution as the answer of a two-stage pipeline, which is all that Lemma~\ref{lem:pipeline} needs. The lemma bounds $\suc_n(\pi^\Phi)$ by $\shal_n(d,s)$ for constants $d$ and $s$ that depend only on the forward pass of $\pi$ and on $\varphi_n$, hence only on $\pi$ and $\Phi$, and the bound on the gap follows from its definition. On $\WP_n$, Lemma~\ref{lem:avgcase} bounds $\shal_n(d,s)$ by $\tfrac{1}{60} + n^{-k}$ for every constant $k$ and all sufficiently large $n$, so every shallow imitator of a model with success $p$ on $\WP_n$ has error at least $p - \tfrac{1}{60} - n^{-k}$. On the other hand, a shallowly written chain caps the accuracy on $\WP_n$ at $\tfrac{1}{60} + n^{-k}$.
\end{proof}

\paragraph{Shortcuts.}
On a maximally serial task no part of the answer can be computed without serial work. On other tasks some part can, and we call such a part a shortcut. A shortcut narrows the answer down to a few candidates, and guessing among them makes it easier for a shallow circuit to get the correct answer.

\begin{proposition}[Shortcuts raise the ceiling]
\label{prop:shortcut}
Suppose that some feature $\phi(y(x))$ of the answer can be computed from the input $x$ by a shallow randomized circuit family $\phi$, and knowing this value leaves at most $m$ candidate answers. Then $\shal_n \geq 1/m$, and a shallow circuit that computes the feature and guesses among the candidates reaches accuracy $1/m$ without any chain.
\end{proposition}

\begin{proof}[Proof of Proposition~\ref{prop:shortcut}]
Let $A$ be the shallow randomized family that computes the feature, so $A(x; r) = \phi(y(x))$ for every input. For each value $q$ in the range of $\phi$, fix an answer $a^*(q)$ that maximizes $\Pr_{x}[\,y(x) = a \mid \phi(y(x)) = q\,]$ over the answers $a$. Since at most $m$ answers are consistent with $q$, this maximum is at least $1/m$. The circuit computes $q = A(x; r)$ and outputs $a^*(q)$ by a lookup table with polynomially many entries, which is a circuit of constant depth. Its accuracy is $\sum_q \Pr[\phi(y(x)) = q] \cdot \Pr[y(x) = a^*(q) \mid \phi(y(x)) = q] \geq 1/m$. The table can be hardwired because our circuit families are non-uniform. The circuit lies in $\mathcal{C}_n(d,s)$ for the depth and the size exponent of $A$ increased by those of the lookup, so $\shal_n(d,s) \geq 1/m$ for these constants, and since the circuit reads only $x$ it reaches this accuracy without any chain, which gives the last statement.
\end{proof}

\paragraph{Shortcuts through a quotient.}
In the case of word problems, the answer feature described in Proposition~\ref{prop:shortcut} is most naturally obtained by taking an appropriate quotient of the group. The formulation below applies to all finite groups, and in the special case of abelianization it gives the expression $1/|[G,G]|$. Table~\ref{tab:prices} reports the resulting ceilings for the groups studied in this paper as well as for the two control groups introduced in Section~\ref{sec:experiments}: the cyclic group $\mathbb{Z}_{60}$ of integers modulo $60$ and the solvable group $A_4 \times \mathbb{Z}_5$. Both control groups have $60$ elements, like $A_5$, and therefore share the same vocabulary and the same guessing baseline.

\begin{proposition}[Shortcuts through a quotient]
\label{prop:quotient}
Let $G$ be a finite group and let $\phi : G \to Q$ be a surjective homomorphism onto a group $Q$ whose word problem lies in $\TCz$. This holds for every commutative group $Q$ and, by \citet{barringtontherien1988}, for every solvable group $Q$. Then, in the log-precision setting, a shallow circuit family reaches accuracy $1/|\ker \phi|$ on the word problem of $G$.
\end{proposition}

\begin{proof}
The circuit applies $\phi$ to every input element, which is a lookup per position. It then computes $q = \phi(g_1) \cdots \phi(g_n) = \phi(y(x))$ with a $\TCz$ circuit for the word problem of $Q$, and it outputs a fixed element of the fiber $\phi^{-1}(q)$. For uniform inputs the product $y(x)$ is uniform on $G$, so conditioned on $\phi(y(x)) = q$ it is uniform on the fiber, which has $|\ker \phi|$ elements. The output is therefore correct with probability $1/|\ker \phi|$.
\end{proof}

\begin{table}[htbp]
\centering
\small
\begin{tabular}{@{}lccccc@{}}
\toprule
Group $G$ & $|G|$ & $[G,G]$ & \shortstack{chance\\$1/|G|$} & \shortstack{price of the\\abelian quotient\\$1/|[G,G]|$} & \shortstack{best shallow\\success (ceiling)} \\
\midrule
$\mathbb{Z}_{60}$ (abelian) & $60$ & $\{e\}$ & $1/60$ & $1$ & $1$ \\
$A_4 \times \mathbb{Z}_5$ (solvable) & $60$ & $V_4$ & $1/60$ & $1/4$ & $1$ \\
$S_5$ & $120$ & $A_5$ & $1/120$ & $1/60$ & $1/60 + o(1)$ \\
$A_5$ (simple) & $60$ & $A_5$ & $1/60$ & $1/60$ & $1/60 + o(1)$ \\
\bottomrule
\end{tabular}
\caption{What a shallow circuit can obtain on the word problem of each group. The price of the abelian quotient is the success of the circuit that computes the image of the product in $G/[G,G]$ and guesses inside the coset (Proposition~\ref{prop:quotient}). The last column is the ceiling $\shal_n(d,s)$ at every fixed depth and size, i.e., the best success of a shallow circuit, up to the error term $n^{-k}$. It equals $1$ for the two solvable groups, because their word problems lie in $\TCz$; these two groups are the controls of Section~\ref{sec:experiments}. For $S_5$ and $A_5$ the last column holds under Assumption~\ref{asm:worstcase} (Proposition~\ref{prop:shortcut} and Lemma~\ref{lem:avgcase}). The ceiling equals the guessing baseline only when the group has no shallow quotient, which is why Section~\ref{sec:exp-toy} instantiates Theorem~\ref{thm:necessity} on $A_5$.}
\label{tab:prices}
\end{table}

For $G = S_5$ with the sign homomorphism, the sign of a product is determined by the parity of the signs of the individual factors, the kernel is $A_5$, and the resulting accuracy is $1/60 = 2/|S_5|$; this matches the lower bound for the $S_5$ case of Proposition~\ref{prop:shortcut} in Section~\ref{sec:exp-toy}. \citet{milesviola2013} already point to the sign map on $S_5$ as the motivation for working over $A_5$. Proposition~\ref{prop:quotient} additionally quantifies the cost, translating this obstruction into an attainable accuracy level for a predictor. Proposition~\ref{prop:quotient} implies that enforcing a ceiling at the pure-guessing rate $1/|G|$ demands that $G$ have no nontrivial quotient whose word problem is in $\TCz$. It is enough that $G$ be perfect: any quotient of a perfect group remains perfect, any nontrivial perfect group is not solvable, and the word problem for a nonsolvable group is $\NCone$-complete \citep{barringtontherien1988}. Simplicity is a stronger condition used by our proof of Lemma~\ref{lem:avgcase}, and $A_5$ satisfies both properties. For the two solvable control groups, we can take $\phi$ to be the identity, so a shallow circuit can compute the full product and the ceiling becomes $1$.

\begin{proof}[Proof of the upper bound for $S_5$ in the instantiation of Proposition~\ref{prop:shortcut}]
We follow the proof of Lemma~\ref{lem:avgcase} with $G = S_5$ and $m = |G| = 120$, and we describe only the changes. Suppose that a polynomial-size randomized $\TCz$ family $\{C_n\}$ has success $s_n \geq \tfrac{1}{60} + \delta$ with $\delta = n^{-k}$ on uniform inputs from $S_5^n$, for all $n$ in an infinite set $N$. Steps~1 to~3 of that proof use only the group axioms, so they carry over unchanged with $\gamma := \delta/(16c)$, where $c \leq |G|$ is the constant from the covering argument below. They yield the class function $\bar D$ with $\bar D(e) = s_n$ and the profile-matching decoder.

In Step~4, suppose that on the event $\mathsf{GOOD}$ the minimizer is some $v \neq y(x)$, and let $u := y(x)\, v^{-1} \neq e$. As before, $|\bar D(b) - \bar D(bw)| \leq 4\gamma$ holds for all $b \in G$ and for every $w$ in the set $K$ of conjugates of $u$ and of $u^{-1}$. Let $H := \langle K \rangle$, which is a normal subgroup of $G$. The proof of Lemma~\ref{lem:covering} uses simplicity only to conclude that $H = G$. Without this hypothesis, the same argument shows that every element of $H$ is a product of at most $c \leq |G|$ elements of $K$. Telescoping as before gives $|\bar D(g) - \bar D(e)| \leq 4c\gamma$ for all $g \in H$. The only normal subgroups of $S_5$ are $\{e\}$, $A_5$, and $S_5$, so $A_5 \subseteq H$. Summing over $g \in A_5$ yields
\[
60\,\big(\bar D(e) - 4c\gamma\big) \;\leq\; \sum_{g \in A_5} \bar D(g) \;\leq\; 1,
\]
and hence $\bar D(e) \leq \tfrac{1}{60} + 4c\gamma = \tfrac{1}{60} + \tfrac{\delta}{4}$, which contradicts $\bar D(e) = s_n \geq \tfrac{1}{60} + \delta$. On $\mathsf{GOOD}$ the unique minimizer is therefore $y(x)$.

Step~5 then yields a deterministic $\TCz/\poly$ family that computes the word problem of $S_5$ exactly for every $n \in N$. Since $A_5$ is a subgroup of $S_5$, the same family computes $\WP_n$ on inputs from $A_5^n$, which contradicts Assumption~\ref{asm:worstcase}.
\end{proof}

\section{Proofs for Section~\ref{sec:depth}}
\label{app:proofs-depth}

\paragraph{Formal setting.}
Given two distributions $P$ and $Q$ over the same finite domain, their total variation distance is defined as $\TV(P, Q) := \max_A |P(A) - Q(A)|$, where the maximum is taken over all events $A$. Let $\pi_c(\cdot \mid x)$ denote the distribution over chains produced by $\pi$’s reasoning model on input $x$. A shallow imitator is a shallow randomized family $\varphi$ that, from an input $x$ together with its internal randomness, outputs a chain; we write $\varphi(x)$ for the resulting output distribution on $x$, and define its imitation error as $\eta_n := \E_x\, \TV\!\big(\varphi(x),\, \pi_c(\cdot \mid x)\big)$. We will use the standard fact that for any function $h:[0,1]$-valued, $|\E_P h - \E_Q h| \le \TV(P,Q)$.

\begin{proof}[Proof of Theorem~\ref{thm:depth}]
The idea is that a shallow imitator is itself an erasing intervention, so the theorem reduces to Theorem~\ref{thm:necessity} through a total variation argument. Let $\varphi$ be a shallow imitator of $\pi$ with imitation error $\eta_n$, and use it as the erasing intervention $\Phi(x, c) := \varphi(x)$. For a fixed input $x$, let $h_x(c) := \pi\big(y(x) \mid x, c\big) \in [0,1]$ be the probability that the answer head outputs the target when it reads the chain $c$. The success of the model on $x$ is $\E_{c \sim \pi_c(\cdot \mid x)}\, h_x(c)$, and the success of $\pi^\Phi$ on $x$ is $\E_{\tilde c \sim \varphi(x)}\, h_x(\tilde c)$. By the property above, the two differ by at most $\TV(\varphi(x), \pi_c(\cdot \mid x))$. Averaging over $x$ gives $\suc_n(\pi^\Phi) \geq \suc_n(\pi) - \eta_n$. Theorem~\ref{thm:necessity} bounds the left-hand side by $\shal_n(d,s)$ for constants $d$ and $s$ that depend only on $\pi$ and $\varphi$, so $\eta_n \geq \suc_n(\pi) - \shal_n(d,s)$, which is the claim. On $\WP_n$, Lemma~\ref{lem:avgcase} bounds $\shal_n(d,s)$ by $\tfrac{1}{60} + n^{-k}$ for every constant $k$ and all sufficiently large $n$, which gives the instantiation stated in Section~\ref{sec:exp-toy}.
\end{proof}

\paragraph{Shallowly written chains.}
Constant-length chains and drafts polished for only a constant number of rounds (Section~\ref{sec:depth}) can be viewed as chains generated by a pipeline as in Definition~\ref{def:pipeline}. Concretely, for a chain of fixed length $T$ the pipeline performs one forward pass per token, and for a draft updated over $q = O(1)$ rounds it performs one forward pass per round. In round $j$, the model runs a forward pass on the input together with the tokens produced in rounds $1, \dots, j-1$, and then samples tokens at some chosen positions. Masked diffusion decoding with $q$ denoising steps fits this second template: at step $j$ it runs a forward pass on the partially unmasked sequence obtained from steps $1, \dots, j-1$ and samples tokens exactly at the positions it reveals. By Lemma~\ref{lem:pipeline}, the output distribution of any such pipeline matches that of a shallow randomized circuit, so it is a shallow imitator with $\eta_n = 0$; likewise, a chain produced by a shallow map is, by definition, a shallow imitator with $\eta_n = 0$. In each of these three settings, Theorem~\ref{thm:depth} yields $\suc_n(\pi) \leq \shal_n$. 

For the mixture case, write $\pi_c(\cdot \mid x) = (1 - \lambda(x))\, \mu_x + \lambda(x)\, \nu_x$, where $\mu_x$ is the output distribution of a shallow map and $\nu_x$ is arbitrary. Then
$\TV(\mu_x, \pi_c(\cdot \mid x)) = \lambda(x)\, \TV(\mu_x, \nu_x) \leq \lambda(x)$,
so the shallow map serves as an imitator with $\eta_n \leq \E_x \lambda(x) = \lambda$. Applying Theorem~\ref{thm:depth} gives $\suc_n(\pi) \leq \shal_n + \lambda$; taking $\lambda$ to be constant in the statement corresponds to the special case where $\lambda(x)$ is constant.

\begin{lemma}[Boilerplate does not enter the bound]
\label{lem:boilerplate}
Suppose the chain of a reasoning model splits as $c = (c_D, c_S)$, where $c_S = f(x, c_D)$ for a shallow randomized map $f$, and write $\pi_{c_D}(\cdot \mid x)$ for the distribution of $c_D$. Then for every shallow imitator $\psi$ of $c_D$, $\suc_n(\pi) \leq \shal_n + \E_x\, \TV\big(\psi(x), \pi_{c_D}(\cdot \mid x)\big)$; that is, Theorem~\ref{thm:depth} holds with $c_D$ in place of $c$.
\end{lemma}

\begin{proof}
Define $\varphi(x)$ by drawing $c_D' \sim \psi(x)$ and outputting $(c_D', f(x, c_D'))$. This is a shallow imitator of $c$, since it composes two shallow randomized maps. The chain of the model is $(c_D, f(x, c_D))$ with $c_D \sim \pi_{c_D}(\cdot \mid x)$, so $\varphi(x)$ and $\pi_c(\cdot \mid x)$ are the images of $\psi(x)$ and $\pi_{c_D}(\cdot \mid x)$ under the same kernel $c_D \mapsto (c_D, f(x, c_D))$, and applying one kernel to two distributions does not increase their total variation distance. Hence $\TV(\varphi(x), \pi_c(\cdot \mid x)) \leq \TV(\psi(x), \pi_{c_D}(\cdot \mid x))$ for every $x$, and Theorem~\ref{thm:depth} applied to $\varphi$ gives the claim.
\end{proof}

\begin{conjecture}[Linear length]
\label{conj:length}
On $\WP_n$, every log-precision transformer whose success tends to $1$ requires a chain of length $T(n) = \Omega(n / \log n)$.
\end{conjecture}

The analogous statement with $T(n) = \Omega(n)$ is a theorem for parity in the unique-hard-attention regime \citep{amiri2025lower}. We leave the version for soft attention and log precision open, which is an instance of Open Problem~5.1 of \citet{amiri2025lower}. Theorem~\ref{thm:depth} shows that the number of sequential steps has to grow without bound. The conjecture concerns the rate.

\section{Proofs for Section~\ref{sec:locality}}
\label{app:proofs-locality}

\paragraph{Formal setting.}
For a transformer $\pi$, an intervention $\Phi$, and a task with target $y$, the shallow decodability of Definition~\ref{def:decodability} is
\[
\decode_\Phi(\pi; d,s) \;:=\; \sup_{D \in \mathcal{C}_n(d,s)}\ \Pr_{x,\ c \sim \pi_c(\cdot \mid x),\ r, r'}\Big[ D\big(x,\, \Phi_n(x, c; r);\, r'\big) = y(x) \Big],
\]
where the supremum ranges over randomized decoders $D$ of depth at most $d$ and size at most $n^s$ in the class of the setting.

\begin{proof}[Proof of Theorem~\ref{thm:locality}]
The intervened answer is drawn as $a \sim \pi(\cdot \mid x, \tilde c)$. By Fact~\ref{fact:tc0}, Fact~\ref{fact:ac0}, and the sampling construction of Appendix~\ref{app:onepass}, the map $(x, \tilde c; r') \mapsto a$ is a randomized circuit of constant depth and polynomial size in the class of the setting, with a depth $d$ and a size exponent $s$ that depend only on $\pi$. It is therefore an admissible decoder $D_0 \in \mathcal{C}_n(d,s)$, and $\Pr[D_0(x, \Phi(x,c;r); r') = y(x)] = \suc_n(\pi^\Phi)$ by the definition of the intervened model. Hence $\suc_n(\pi^\Phi) \leq \decode_\Phi(\pi; d,s)$. This argument uses neither the task nor the input distribution. For an erasing intervention, a decoder in $\mathcal{C}_n(d,s)$ applied to $(x, \varphi_n(x; r))$ is a pipeline with two stages and hence a shallow randomized circuit (Lemma~\ref{lem:pipeline}), whose depth and size exponent $d'$ and $s'$ depend only on $d$, $s$, and $\varphi_n$, so $\decode_\Phi(\pi; d,s) \leq \shal_n(d',s')$; in particular a decoder that ignores the chain has accuracy at most $\shal_n$, whereas $\decode(\pi) \geq \suc_n(\pi) = p$ by the first statement. Combined with the first statement, the erasing bound recovers Theorem~\ref{thm:necessity}. On $\WP_n$, no decoder reads the answer from an erased chain with accuracy above $\tfrac{1}{60} + n^{-k}$ (Lemma~\ref{lem:avgcase}), whereas the answer head of a model with success $p$ reads it from the model's own chain with accuracy $p$; the chain raises the shallow decodability from chance to $p$.
\end{proof}

\paragraph{Shuffling and the bag.}
Fix a way to segment the chain into units—either tokens or sentences, where sentences are split at sentence-final punctuation and line breaks. This segmentation can be computed by a shallow circuit, since it merely needs to identify the delimiters. The shuffle intervention $\Phi^{\mathrm{shuf}}$ returns the units of $c$ arranged according to a uniformly random permutation $\sigma$. Let $\bag(c)$ denote the multiset of units in $c$, and define $\decode_{\bag}(\pi)$ as the intervention that produces a canonical encoding of $\bag(c)$, e.g., by listing the units in sorted order.
\begin{lemma}[Shuffle sees the bag]
\label{lem:bag}
For either segmentation, in the log-precision setting, for every $d$ and $s$ there are constants $d'$ and $s'$ such that $\decode_{\Phi^{\mathrm{shuf}}}(\pi; d,s) - \negl(n) \;\leq\; \decode_{\bag}(\pi; d',s')$ and $\decode_{\bag}(\pi; d,s) \;\leq\; \decode_{\Phi^{\mathrm{shuf}}}(\pi; d',s')$.
\end{lemma}

\begin{proof}
We first show the second inequality. Given $\sigma(c)$, the canonical sorted sequence is a deterministic $\TCz$ computation that depends only on the bag, since the ranks follow from iterated addition and comparisons, so any bag decoder in $\mathcal{C}_n(d,s)$ becomes a shuffle decoder with identical success, at a constant additional depth and a polynomial additional size, which fixes $d'$ and $s'$. We then show the first inequality. Given the bag, assign each unit an independent random key of $C \log n$ bits and order the units by key. Conditioned on all keys being distinct, which fails with probability at most $T^2 2^{-C\log n} = \negl(n)$ for a large constant $C$, the resulting ordering is exactly uniform, that is, it has the law of $\sigma(c)$ given the bag. Running any shuffle decoder in $\mathcal{C}_n(d,s)$ on this simulated input loses at most the failure probability, again at a constant additional depth and a polynomial additional size.
\end{proof}

Together with Theorem~\ref{thm:locality}, the lemma gives $\suc_n(\pi^{\Phi^{\mathrm{shuf}}}) \leq \decode_{\bag}(\pi) + \negl(n)$, which is the statement about shuffling in Section~\ref{sec:locality}.

\begin{proposition}[Fragile to erasure and robust to shuffling]
\label{prop:dissoc}
There is a polynomial-size log-precision transformer $\pi^\star$ with $\suc_n(\pi^\star) = 1$ on $\WP_n$ whose load-bearing gap is $0$ under shuffling. Under Assumption~\ref{asm:worstcase}, its gap is at least $1 - \tfrac{1}{60} - n^{-k}$ under every erasing intervention, for every constant $k$ and all sufficiently large $n$.
\end{proposition}

\begin{proof}
On input $x$, the transformer $\pi^\star$ deterministically emits $c = (P_1, \dots, P_n, y, y, \dots, y)$, where $P_t = g_1 \cdots g_t$ are the prefix products and the final block repeats $P_n = y(x)$ exactly $2n+1$ times, so $T = 3n + 1$. Its answer head outputs the plurality element of the chain-token sub-vocabulary in its context. Such a generator exists, because computing $P_t$ from $(P_{t-1}, g_t)$ is a single multiplication in a finite group, which one attention step implements by retrieving the previous chain token and the aligned input token, followed by a constant lookup. This is a special case of the step simulation of \citet[Thm.~3.3]{li2024cot}, and repeating the last value is trivial. Such an answer head exists as well. With one-hot value vectors, uniform attention over the chain segment yields the exact frequency vector of the $60$ symbols, where segment embeddings distinguish chain tokens from input tokens, and each entry is a multiple of $1/T$ and representable at $O(\log n)$ bits. Uniform attention scores do not depend on the order of the tokens, which is exactly why the head is robust to shuffling. The plurality symbol leads every other symbol by a margin of at least $(n+1)/(3n+1) > 1/3$, so a constant-size feed-forward argmax with constant separation suffices and tolerates the rounding. The transformer is correct on its own chain and on every permutation of it. In the unshuffled chain, and in any permutation of it, the multiplicity of $y(x)$ is at least $2n+1$, while any other symbol occurs at most $n$ times among the prefixes, so the plurality is $y(x)$ deterministically, $\suc_n(\pi^\star) = \suc_n((\pi^\star)^{\Phi^{\mathrm{shuf}}}) = 1$, and the gap under shuffling is $0$. Finally, $\pi^\star$ is fragile to erasure. It is a polynomial-size log-precision transformer, so Theorem~\ref{thm:necessity} applies to it as stated, and Lemma~\ref{lem:avgcase} bounds the ceiling of $\WP_n$. For every erasing intervention, the intervened success is therefore at most $\tfrac{1}{60} + n^{-k}$.
\end{proof}

\paragraph{Fragility under content-dependent interventions.}
Theorem~\ref{thm:locality} limits what the answer head can contribute, but it does not restrict how the answer head acts on chains the model would never generate on its own, such as a permuted or truncated chain. Consequently, a reduction in accuracy under such an intervention reflects the answer head rather than the chain itself, while an erasing intervention guarantees a drop for any model by Theorem~\ref{thm:necessity}. The next proposition highlights this contrast in its most extreme form. Section~\ref{sec:experiments} leverages it to interpret the interventions on open models.

\begin{proposition}[Fragility does not mean that the answer is absent]
\label{prop:converse}
There are two transformers on $\WP_n$ that generate exactly the same chains and both have success $1$. The first keeps success $1$ when its chain is shuffled. The second falls to $\tfrac{1}{60} + O(1/n)$ under the same intervention, although a decoder reads the answer from its shuffled chain with accuracy $1$.
\end{proposition}

\begin{proof}[Proof of Proposition~\ref{prop:converse}]
The first transformer is $\pi^\star$ from Proposition~\ref{prop:dissoc}. Let $\pi'$ agree with $\pi^\star$ on generation (hence identical chains and identical success $1$) but differ in the behavior of the answer head on contexts outside the support of generation. On any context whose chain segment is not of the canonical form, that is, prefix products followed by a constant block, $\pi'$ outputs a fixed symbol $z \in G$. Both behaviors are realizable by polynomial-size log-precision transformers. The two heads are gated by a constant-depth check of the canonical form, which verifies $c_t = c_{t-1}\, g_t$ for $t \leq n$ with the convention $c_0 := e$, and $c_t = c_{t-1}$ for $t > n$, and this check is an $\mathsf{NC}^0$ test per position aggregated by an AND. Under $\Phi^{\mathrm{shuf}}$, passing the canonical-form check pins down the whole sequence, because the check forces $c_1 = g_1$ and each subsequent token by induction, so a shuffled chain passes if and only if the permutation reproduces the canonical sequence exactly. For fixed $x$, this happens with probability $\prod_g m_g! \,/\, (3n+1)!$ under a uniform permutation, where $m_g$ are the multiplicities of the chain's token values. If the bag contains at least two distinct values, $\prod_g m_g! \leq m_{\max}!\,(3n+1-m_{\max})!$, so the probability is at most $1/\binom{3n+1}{m_{\max}} \leq 1/(3n+1)$; the bag is single-valued only on the degenerate instances $x = (y, e, \dots, e)$, a $60^{1-n}$ fraction of inputs. Hence $\suc_n(\pi'^{\Phi^{\mathrm{shuf}}}) \leq \tfrac1{60} + \tfrac{1}{3n+1} + 60^{1-n} = \tfrac1{60} + O(1/n)$, while $\suc_n((\pi^\star)^{\Phi^{\mathrm{shuf}}}) = 1$. The answer head of $\pi^\star$ is a decoder that reads the answer from the shuffled chain of $\pi'$ with accuracy $1$, because the two transformers generate the same chains.
\end{proof}

\begin{table}[htbp]
\centering
\footnotesize
\begin{tabular}{@{}>{\raggedright\arraybackslash}p{2.35cm}>{\raggedright\arraybackslash}p{2.05cm}>{\raggedright\arraybackslash}p{4.15cm}>{\raggedright\arraybackslash}p{4.15cm}@{}}
\toprule
Intervention & The answer head still sees & If the accuracy stays high, we learn that & If the accuracy drops, we learn that \\
\midrule
Erasing (filler, random tokens, template, empty chain)
 & nothing from the chain
 & the model can answer these inputs without its chain, so a shortcut exists
 & the content of the chain was needed. On a maximally serial task the drop is forced and goes all the way to the guessing baseline (Theorem~\ref{thm:necessity}) \\
\addlinespace
Truncation (keep a prefix)
 & the first part of the chain
 & the answer was already determined at that point of the chain
 & this answer head cannot finish the remaining work in one pass. We learn nothing about the chain itself \\
\addlinespace
Shuffle (tokens or sentences)
 & the bag of tokens or of sentences
 & the answer is readable from the bag, so the chain states it in a way that does not depend on order
 & nothing about the chain. Two policies with identical chains can react in opposite ways (Proposition~\ref{prop:converse}) \\
\addlinespace
Re-encoding (relabeling, paraphrase)
 & all of the information, in a different form
 & the answer head relies on features that survive the re-encoding
 & nothing about the chain. The information is unchanged, so the drop only reflects the answer head \\
\bottomrule
\end{tabular}
\caption{How to read an intervention experiment. By Theorem~\ref{thm:locality}, the accuracy after an intervention is bounded by what a shallow decoder can read from the part of the chain that the intervention keeps. High accuracy is therefore informative for every intervention. A drop in accuracy is informative about the chain only for erasing interventions.}
\label{tab:semantics}
\end{table}

\section{Experimental setup}
\label{app:exp}

\paragraph{Small transformers.}
We train small transformer models with $L \in \{2, 4, 8\}$ layers at input lengths $n$ ranging from $8$ to $256$ on four word problems. The tasks are $\WP_n$, the word problem of $A_5$; the word problem of $S_5$ (Proposition~\ref{prop:shortcut} in Appendix~\ref{app:proofs-necessity}); and the word problems of the two solvable control groups $\mathbb{Z}_{60}$ and $A_4 \times \mathbb{Z}_5$ of Table~\ref{tab:prices} in Appendix~\ref{app:proofs-necessity}, which have ceiling $1$ and the same vocabulary and baseline as $A_5$. Two training regimes are applied to each task. The first trains the model to generate the chain of running products defined in Section~\ref{sec:exp-toy} before answering. The second trains the model to answer in a single forward pass without generating a chain. A chainless model is a pipeline with one forward pass, so Lemma~\ref{lem:pipeline} caps its success at the ceiling of its task, which is listed in Table~\ref{tab:prices} for the four groups.

All models are decoder-only transformers with rotary positional embeddings, hidden size $512$, $8$ attention heads, and $L \in \{2, 4, 8\}$ layers. They are trained from scratch with batch size $256$ for $6{,}000$ steps when $n \leq 32$, $10{,}000$ steps when $n \in \{64, 128\}$, and $16{,}000$ steps when $n = 256$. Inputs consist of uniformly sampled group-element sequences, represented with one token per element, and evaluation is performed on $10{,}000$ new inputs using a fixed evaluation seed. In the chain-trained setting, the model is trained to output the running products $P_t = g_1 \cdots g_t$ as an explicit chain and then produce the final answer; erasing interventions in this setting replace that chain with filler tokens, uniformly random tokens, or a fixed template. In the no-chain setting, the model outputs the answer in a single forward pass (thus forming a one-pass pipeline under Definition~\ref{def:pipeline}) and is trained with the correct running product as the target at every position, providing the strongest supervision available without an explicit chain. Parity experiments follow Definition~\ref{def:parity}, and control groups are as specified in Table~\ref{tab:prices}. Most core cells use three random seeds, while the $S_5$ cells at $n \in \{32, 64, 128\}$ use eight. For each run, the tracked prefix $t^\ast$ is defined as the largest index $t$ such that per-position accuracy is at least $0.9$ for every position up through $t$. At $n = 256$, the fixed training budget becomes insufficient even for commutative controls (whose maximum accuracy is $1$), leading to reduced accuracy; accordingly, plateau comparisons in the main text focus on $n \leq 128$. The no-chain runs were introduced after the main program was finalized and were labeled retrospectively in the released logs.

\paragraph{Chain formats, relabeling, and the bag decoder.}
For Figure~\ref{fig:small}f, models with $L = 4$ are trained on $A_5$ at $n \in \{16, 32, 64\}$ with three seeds on one of three chains: the running products $P_1, \dots, P_n$; the running products with each product written twice; and the running products followed by $2n + 1$ copies of $P_n = y(x)$, the chain of Proposition~\ref{prop:dissoc}. Besides the three erasing interventions, we apply a token shuffle, which permutes the chain tokens uniformly at random, and a relabeling, which applies one fixed random bijection of the group to every chain token. The bag decoder is the better of a logistic regression and a network with one hidden layer of $128$ units, both fit on the vector that counts how often each group element occurs in the chain the model generated, on $70\%$ of the evaluation inputs, and scored on the remaining $30\%$. Its accuracy is a lower bound on the bag decodability of Lemma~\ref{lem:bag}.

\paragraph{Chain length.}
The chain that writes every $k$-th running product lists $P_k, P_{2k}, \dots$ and always ends with $P_n$, so it has $T = \lceil n/k \rceil$ steps and its last token is the answer; $T = n$ is the full chain. Models with $L = 4$ are trained at $n \in \{16, 32, 64, 128\}$ for every $T \in \{1, 2, 4, \dots, n\}$, with three seeds each.

\paragraph{Reinforcement learning on the word problem.}
Models with $L = 4$ are trained on $A_5$ at $n = 32$ on the chain for $3{,}000$ supervised steps, which leaves their accuracy well below one, and then for $1{,}500$ steps with GRPO, RLOO, or PPO, with learning rate $10^{-4}$, a KL coefficient of $0.03$, and reward $1$ for a correct final answer and $0$ otherwise. We first ran RL after training on the chain to accuracy $1$; every sample then received reward $1$, the policy gradient vanished, and RL left the model unchanged, so we report only the protocol above.

\paragraph{Open reasoning models.}
We test DeepSeek-R1-Distill-Qwen-1.5B, -7B, and -14B, as well as Qwen3-32B, on MATH-500 and additionally evaluate Qwen3-32B on AIME24 and AIME25, using four generations per problem and seed $1234$ for both the engine and the sampler. Checkpoints and datasets were downloaded from the Hugging Face hub without pinning specific revisions. The released records specify the chat templates used to format the prompts, and the latest hub modification of those templates predates our executions. The model initially produces its solution within a think block, sampled with temperature $0.6$, top-$p$ $0.95$, and top-$k$ $20$, using a chain budget of $16{,}384$ tokens for models with at least 14B parameters and $8{,}192$ tokens for smaller models. Our intervention overwrites the think-block content, then closes the block, after which decoding proceeds greedily under a forced suffix ending in \texttt{Final Answer: \textbackslash boxed\{}, for up to $64$ tokens, while a logit bias prevents the think tag from being generated so the block cannot be reopened. Consequently, everything generated after the intervened chain corresponds to the answer head in Section~\ref{sec:prelim}, i.e., a single forward pass on the input plus the intervened chain; Remark~\ref{rem:multitoken} discusses the case where the answer spans multiple tokens.

We define six channels. \emph{Filler} substitutes the chain with a neutral token sequence of the same length. \emph{Random tokens} swaps it for uniformly sampled tokens with matched length. \emph{Template} repeats the question text until it reaches the chain's length. \emph{Sentence shuffle} permutes the chain's sentences uniformly at random: the chain is segmented at sentence-final punctuation and line breaks, and each sentence remains unchanged. \emph{Token shuffle} randomly permutes all tokens in the chain. \emph{Truncation} retains the specified fraction of the chain's tokens. The first three are erasing as defined in Section~\ref{sec:necessity}, while the last three are content-dependent as defined in Section~\ref{sec:locality}; Table~\ref{tab:semantics} describes what each one certifies. For each problem, we also note whether the ground-truth answer appears verbatim anywhere in the intervened chain (referred to in the main text as the answer being present); for randomized channels, this is evaluated on an independent sample from the same channel. Accuracy is computed as exact match on the forced answer. The released records have two caveats. We did not store the generated chain texts, only aggregated per-cell metrics, so quantities that require the raw chains (e.g., sentences per chain) cannot be recomputed. Additionally, for the three distilled models whose chat template ends the prompt with an open think tag followed by a newline, the intervened context re-adds the tag but omits the newline. This occurs consistently across all channels, including the identity condition, so within-cell comparisons remain consistent, but each intervened context differs from the clean run by that single character.

\paragraph{Open models on the word problem.}
We render each word of $\WP_n$ in natural language. The $60$ elements of $A_5$ receive fixed names of orientations of an object, such as \emph{Upright-North}. The prompt states that the object starts in the orientation named by the identity, lists the rotations $g_1, \dots, g_n$ by name, and gives, for every rotation that occurs in the word, one line with the orientation reached from each orientation, that is, the corresponding column of the multiplication table. The answer is the final orientation, so $b = 1/60$. The controls $\mathbb{Z}_{60}$ and $A_4 \times \mathbb{Z}_5$ use the same names and the same template, so only the multiplication table changes. We also include parity, rendered as a coin that is flipped or left alone at each step ($b = 1/2$), and a retrieval task that counts marked items in a context whose length grows with $n$. We use the post-trained Qwen3 models with 4B, 8B, 14B, and 32B parameters in thinking mode, $200$ words per task and length, and the sampling settings above with seed $1234$, with at most $2{,}048$ chain tokens. For the budget sweep, the model generates at most $T \in \{0, 32, 64, \dots, 2{,}048\}$ thinking tokens, the block is closed, and the answer is forced with the suffix \texttt{The final orientation is:} for up to $24$ tokens, with the think tag banned; $T = 0$ closes the block immediately. For the erasing intervention, the full chain is replaced by the same number of copies of one neutral token, and the accuracy with the full chain is measured with the same forced suffix.

\paragraph{Reinforcement-learning checkpoints.}
We train Qwen3 models with GRPO \citep{shao2024deepseekmath} in verl \citep{sheng2024hybridflow}, with a token-level loss, on mathematical problems from DeepScaleR and from the harder part of DeepMath, after removing every problem that shares a word $n$-gram with MATH-500, AIME24, or AIME25. The reward is $1$ for a correct final answer, or a pseudo-random bit that does not depend on the answer \citep{shao2025spurious}. The correct reward is applied to the post-trained Qwen3-4B, 8B, and 14B, and the random reward to the post-trained Qwen3-4B and 8B and to Qwen3-4B-Base and Qwen3-8B-Base; one checkpoint of Qwen3-4B is trained with a reward for the answer format only. Table~\ref{tab:app-checkpoints} lists the $22$ checkpoints we evaluated and their training steps. They are evaluated on MATH-500 with the protocol above, on $500$ problems with four generations each, except for four checkpoints evaluated on $200$ problems and two on $64$ problems with two generations. On three checkpoints the forced answer ignores the chain: every intervention gives the same accuracy, and the answer is never present in the chain. We exclude these three from Figure~\ref{fig:open}d.

\section{Additional experimental results}
\label{app:more-exp}

\begin{figure}[t]
\centering
\includegraphics[width=\textwidth]{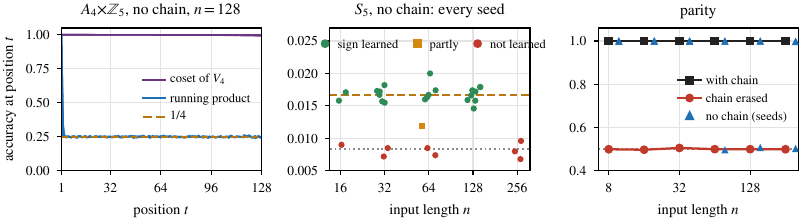}
\caption{What small transformers learn without a chain, and parity ($L = 4$). \emph{Left:} on $A_4 \times \mathbb{Z}_5$ at $n = 128$, chainless models predict the coset of the running product modulo $V_4 = [G, G]$ at every position and are right about the product itself a quarter of the time (three seeds). \emph{Middle:} every chainless seed on $S_5$; seeds that learn the sign of the product land at the ceiling $2b = 1/60$ (dashed), the others at $b = 1/120$ (dotted). \emph{Right:} on parity, chain-trained models fall to $1/2$ when their chain is erased, while chainless models with dense supervision (one triangle per seed) solve it in $15$ of $18$ seeds.}
\label{fig:app-nochain}
\end{figure}

\paragraph{What chainless models learn.}
Figure~\ref{fig:app-nochain} (left and middle) looks inside the chainless models of Section~\ref{sec:exp-toy}. On $A_4 \times \mathbb{Z}_5$, six reruns at $n = 64$ and $128$ predict the coset of the running product modulo $V_4 = [G, G]$ with accuracy $0.999$ to $1.000$, averaged over positions $t \geq 16$, while their accuracy on the product itself is $0.249$ to $0.251$. They compute the image of the product in the abelian quotient and guess inside the coset, which is the circuit of Proposition~\ref{prop:quotient}, and they never find the shallow circuit for the rest of the product. On $S_5$, we count a seed as having learned the sign when its sign accuracy at the answer is at least $0.95$ and its accuracy lies within three standard errors of $1/60$, and as not having learned it when its sign accuracy is at most $0.55$ and its accuracy is at most $0.012$. At $n = 16$, $32$, $64$, $128$, and $256$, the sign is learned in $2$ of $3$, $6$ of $8$, $5$ of $8$, $8$ of $8$, and $0$ of $3$ seeds, and one seed at $n = 64$ is in between. Seeds that learn it score $0.0146$ to $0.0200$, and the others score $0.0068$ to $0.0096$. A hand-coded shallow circuit that computes the sign as the parity of the signs of the factors and guesses inside the coset reaches $0.016$ to $0.021$ at $n = 8$ to $64$, the value $2b$ of Proposition~\ref{prop:quotient}, while a uniform guess on $A_5$ reaches $0.014$ to $0.021$. With the final answer as the only target, training does not find the sign at $n = 16$ to $64$: a model trained to predict only the sign stays at $0.48$ to $0.50$, and so do a linear probe for the sign on the hidden state of a chainless model ($0.47$ to $0.53$), training for $120{,}000$ steps ($0.46$ to $0.51$), and pretraining on the sign before training on the product ($0.48$ to $0.52$). The sign is a parity of the input, and only the running targets of dense supervision let training find it.

\paragraph{Parity.}
Parity (Definition~\ref{def:parity}) is maximally serial at constant precision (Lemma~\ref{lem:parity}) but not serial at log precision, where a majority gate computes it (Appendix~\ref{app:constprec}). Chain-trained models reach $1.000$ at every length from $8$ to $256$ and fall to between $0.493$ and $0.509$ when their chain is erased. Chainless models with dense supervision reach at least $0.997$ in $15$ of $18$ seeds and stay at $0.50$ in the other three (Figure~\ref{fig:app-nochain}, right), consistent with the log-precision setting. As on the two non-serial groups, the drop under erasing comes from the answer head of the chain-trained models and not from the task.

\paragraph{Models trained on the answer alone or with filler tokens.}
Besides the chain and dense supervision, we trained models that answer immediately with the final answer as the only target, and models whose chain positions hold a fixed filler token. On all four groups and at every length, both stay within $0.0012$ of the baseline in the mean over seeds, including $\mathbb{Z}_{60}$ at $n = 8$, which chainless models with dense supervision solve; on parity they reach $0.66$ at $n = 8$ and $0.50$ beyond. These regimes learn nothing, so they neither support nor test the theorems, and we report them only for completeness.

\paragraph{Witnesses and twins.}
Table~\ref{tab:app-witness} realizes Propositions~\ref{prop:dissoc} and~\ref{prop:converse}. A hand-coded answer head that outputs the most frequent chain token, run on the chain of Proposition~\ref{prop:dissoc}, keeps $1.000$ under a token shuffle and falls to the baseline or below under erasing. Models trained on this chain survive the shuffle across $12$ architectures ($0.690$ to $1.000$ at $n = 32$, with hidden size $256$ or $512$, $L \in \{2, 4, 8\}$, and learned or rotary positions) and at every length from $18$ to $56$ ($0.696$ to $0.926$), always below the trained bag decoder. For Proposition~\ref{prop:converse}, two hand-coded heads read identical chains: the head that takes the plurality keeps $1.000$ under a shuffle, and the head gated on the canonical form of the chain falls to $0.015$ to $0.018$, although the plurality head reads the answer from the same shuffled chains. Two students cloned from the chains of one trained teacher, one with the teacher's architecture and one with eight layers, react to a shuffle in the same way ($0.032$ to $0.077$, below the bag decoder at $0.034$ to $0.108$), so training does not by itself produce twins that diverge.

\begin{table}[t]
\centering
\footnotesize
\caption{Witnesses for Propositions~\ref{prop:dissoc} and~\ref{prop:converse} on the word problem of $A_5$. Ranges over input lengths, architectures, and the three erasing interventions (means over three seeds, or over three evaluation seeds for the hand-coded heads). The bag decoder is trained on the counts of the chain tokens; the hand-coded twins read identical chains, so the plurality head is itself a decoder with accuracy $1$ on the shuffled chains of the gated head.}
\label{tab:app-witness}
\begin{tabular}{@{}p{4.35cm}ccccc@{}}
\toprule
setting & $n$ & with chain & chain erased & tokens shuffled & bag decoder \\
\midrule
Hand-coded plurality head on the chain of Prop.~\ref{prop:dissoc} & 16--64 & 1.000 & 0.000--0.018 & 1.000 & 1.000 \\
Trained on that chain, 12 architectures & 32 & 1.000 & 0.015--0.019 & 0.690--1.000 & 1.000 \\
Trained on that chain, $L = 4$ & 18--56 & 1.000 & 0.014--0.019 & 0.696--0.926 & 1.000 \\
Twin with plurality head (Prop.~\ref{prop:converse}) & 16--64 & 1.000 & 0.000--0.021 & 1.000 & -- \\
Twin with gated head (Prop.~\ref{prop:converse}) & 16--64 & 1.000 & 0.015--0.018 & 0.015--0.018 & -- \\
Student cloned from a trained teacher, same architecture & 16--64 & 1.000 & 0.009--0.019 & 0.032--0.077 & 0.034--0.108 \\
Student cloned from the same teacher, $L = 8$ & 16--64 & 0.999--1.000 & 0.013--0.019 & 0.032--0.077 & 0.034--0.108 \\
\bottomrule
\end{tabular}
\end{table}

\paragraph{Open models on the word problem and its controls.}
Table~\ref{tab:app-wp-controls} gives the full chain and the erased chain for every model, task, and length. Replacing the chain by filler drops every model to the baseline on $A_5$, and equally on $\mathbb{Z}_{60}$ and $A_4 \times \mathbb{Z}_5$, whose ceiling is $1$, while with the full chain the three groups are solved about equally often. The models compute all three products by stepping through the word and do not use the shortcut that the solvable groups admit. On parity the chain helps every model except Qwen3-8B, and filler returns them to about $1/2$; on the retrieval task, filler costs Qwen3-32B $0.02$ at $n = 4$ and $0.80$ at $n = 16$. Before running this experiment we expected the drop under filler to be near zero on the solvable groups and on the retrieval task, and largest on $A_5$ and parity. That expectation did not hold. The theorems bound the accuracy after erasing from above and say nothing about how far a model that relies on its chain falls when the ceiling is one.

\begin{table}[t]
\centering
\footnotesize
\caption{Open models on the word problem of $A_5$ and on its controls, $200$ problems per cell: accuracy with the full chain / with the chain replaced by filler of the same length. The baseline is $1/60$ on the three groups and $1/2$ on parity; the retrieval task counts marked items in a context that grows with $n$.}
\label{tab:app-wp-controls}
\begin{tabular}{llcccc}
\toprule
task & model & $n = 4$ & $n = 6$ & $n = 8$ & $n = 16$ \\
\midrule
$A_5$ & Qwen3-4B & 0.12 / 0.01 & 0.03 / 0.01 & 0.03 / 0.01 & 0.02 / 0.03 \\
 & Qwen3-8B & 0.29 / 0.01 & 0.09 / 0.03 & 0.07 / 0.01 & 0.03 / 0.01 \\
 & Qwen3-14B & 0.45 / 0.02 & 0.26 / 0.01 & 0.17 / 0.00 & 0.02 / 0.01 \\
 & Qwen3-32B & 0.62 / 0.01 & 0.37 / 0.02 & 0.29 / 0.01 & 0.04 / 0.02 \\
\midrule
$\mathbb{Z}_{60}$ & Qwen3-4B & 0.14 / 0.03 & 0.03 / 0.01 & 0.02 / 0.01 & 0.01 / 0.01 \\
 & Qwen3-8B & 0.35 / 0.03 & 0.12 / 0.01 & 0.04 / 0.03 & 0.02 / 0.01 \\
 & Qwen3-14B & 0.47 / 0.01 & 0.23 / 0.01 & 0.07 / 0.03 & 0.01 / 0.01 \\
 & Qwen3-32B & 0.60 / 0.01 & 0.32 / 0.02 & 0.14 / 0.01 & 0.01 / 0.01 \\
\midrule
$A_4{\times}\mathbb{Z}_5$ & Qwen3-4B & 0.12 / 0.01 & 0.04 / 0.03 & 0.04 / 0.03 & 0.03 / 0.01 \\
 & Qwen3-8B & 0.34 / 0.03 & 0.15 / 0.03 & 0.04 / 0.04 & 0.03 / 0.01 \\
 & Qwen3-14B & 0.52 / 0.01 & 0.30 / 0.03 & 0.10 / 0.01 & 0.02 / 0.01 \\
 & Qwen3-32B & 0.65 / 0.03 & 0.45 / 0.03 & 0.24 / 0.01 & 0.04 / 0.03 \\
\midrule
parity & Qwen3-4B & 0.58 / 0.46 & 0.76 / 0.49 & 0.81 / 0.47 & 0.70 / 0.49 \\
 & Qwen3-8B & 0.47 / 0.46 & 0.49 / 0.49 & 0.47 / 0.47 & 0.56 / 0.49 \\
 & Qwen3-14B & 0.99 / 0.46 & 0.95 / 0.49 & 0.93 / 0.47 & 0.51 / 0.49 \\
 & Qwen3-32B & 0.64 / 0.56 & 0.73 / 0.51 & 0.72 / 0.44 & 0.91 / 0.48 \\
\midrule
retrieval & Qwen3-4B & 0.82 / 0.39 & 0.81 / 0.32 & 0.66 / 0.29 & 0.62 / 0.13 \\
 & Qwen3-8B & 0.81 / 0.34 & 0.97 / 0.48 & 0.99 / 0.34 & 0.98 / 0.13 \\
 & Qwen3-14B & 1.00 / 0.56 & 0.99 / 0.46 & 0.96 / 0.45 & 0.97 / 0.20 \\
 & Qwen3-32B & 0.99 / 0.97 & 0.99 / 0.65 & 0.97 / 0.38 & 0.97 / 0.17 \\
\bottomrule
\end{tabular}
\end{table}

\paragraph{Reinforcement-learning checkpoints.}
Table~\ref{tab:app-checkpoints} lists the accuracy of every checkpoint of Figure~\ref{fig:open}d under every intervention, and the three checkpoints we excluded.

\begin{table}[t]
\centering
\scriptsize
\setlength{\tabcolsep}{4pt}
\caption{The $22$ Qwen3 checkpoints trained by GRPO that we evaluated on MATH-500 (records are problems times generations): accuracy with the full chain and after each intervention, all with the forced answer. $^\dagger$~Excluded from Figure~\ref{fig:open}d: the forced answer of these checkpoints ignores the chain, so every intervention gives the same accuracy, and the answer is never present in the chain.}
\label{tab:app-checkpoints}
\begin{tabular}{llrrccccccc}
\toprule
start model & reward & step & records & full chain & filler & random & template & empty & sent.\ shuffle & token shuffle \\
\midrule
Qwen3-4B & correct & 25 & 128 & 0.891 & 0.258 & 0.266 & 0.297 & 0.297 & 0.891 & 0.328 \\
Qwen3-4B & correct & 50 & 128 & 0.953 & 0.297 & 0.273 & 0.336 & 0.312 & 0.914 & 0.375 \\
Qwen3-4B & correct & 425 & 2000 & 0.573 & 0.205 & 0.206 & 0.242 & 0.218 & 0.566 & 0.197 \\
Qwen3-8B & correct & 45 & 2000 & 0.853 & 0.266 & 0.244 & 0.296 & 0.304 & 0.839 & 0.271 \\
Qwen3-8B & correct & 255 & 2000 & 0.782 & 0.188 & 0.157 & 0.268 & 0.274 & 0.690 & 0.079 \\
Qwen3-14B & correct & 110 & 2000 & 0.874 & 0.316 & 0.300 & 0.334 & 0.340 & 0.861 & 0.332 \\
Qwen3-4B$^\dagger$ & format only & 750 & 800 & 0.010 & 0.010 & 0.010 & 0.010 & 0.010 & 0.010 & 0.010 \\
Qwen3-4B & random & 25 & 800 & 0.870 & 0.226 & 0.225 & 0.276 & 0.280 & 0.844 & 0.264 \\
Qwen3-4B$^\dagger$ & random & 475 & 800 & 0.020 & 0.020 & 0.020 & 0.020 & 0.020 & 0.020 & 0.020 \\
Qwen3-4B-Base & random & 25 & 2000 & 0.573 & 0.258 & 0.262 & 0.286 & 0.274 & 0.572 & 0.249 \\
Qwen3-4B-Base & random & 50 & 2000 & 0.557 & 0.260 & 0.254 & 0.289 & 0.266 & 0.589 & 0.262 \\
Qwen3-4B-Base & random & 100 & 2000 & 0.607 & 0.210 & 0.247 & 0.297 & 0.276 & 0.602 & 0.265 \\
Qwen3-4B-Base & random & 150 & 2000 & 0.610 & 0.212 & 0.243 & 0.288 & 0.278 & 0.594 & 0.261 \\
Qwen3-4B-Base & random & 200 & 2000 & 0.602 & 0.211 & 0.243 & 0.290 & 0.276 & 0.598 & 0.266 \\
Qwen3-4B-Base & random & 300 & 2000 & 0.594 & 0.222 & 0.253 & 0.292 & 0.274 & 0.590 & 0.258 \\
Qwen3-8B$^\dagger$ & random & 285 & 800 & 0.325 & 0.325 & 0.325 & 0.325 & 0.325 & 0.325 & 0.325 \\
Qwen3-8B-Base & random & 30 & 2000 & 0.615 & 0.306 & 0.306 & 0.317 & 0.320 & 0.594 & 0.306 \\
Qwen3-8B-Base & random & 45 & 2000 & 0.627 & 0.306 & 0.298 & 0.317 & 0.312 & 0.609 & 0.302 \\
Qwen3-8B-Base & random & 105 & 2000 & 0.618 & 0.326 & 0.316 & 0.335 & 0.336 & 0.598 & 0.303 \\
Qwen3-8B-Base & random & 150 & 2000 & 0.695 & 0.280 & 0.290 & 0.340 & 0.304 & 0.673 & 0.293 \\
Qwen3-8B-Base & random & 195 & 2000 & 0.669 & 0.273 & 0.279 & 0.348 & 0.300 & 0.638 & 0.303 \\
Qwen3-8B-Base & random & 300 & 2000 & 0.644 & 0.286 & 0.288 & 0.327 & 0.310 & 0.616 & 0.280 \\
\bottomrule
\end{tabular}
\end{table}

\end{document}